\pdfoutput=1
\documentclass{article}
\usepackage{iclr2027_conference,times}
\usepackage{amsmath,amsfonts,bm}
\usepackage{amsmath,amssymb,amsthm,booktabs,graphicx,array,xcolor}
\usepackage{fix-cm,tikz}
\usetikzlibrary{arrows.meta,calc,shapes.geometric,shapes.misc,decorations.text,shadings,shadows.blur}
\usepackage{algorithm,algpseudocode}
\usepackage{hyperref,url,needspace}
\hypersetup{hidelinks}
\definecolor{chrono}{HTML}{167C80}
\definecolor{draftcolor}{HTML}{87652B}

\usepackage{glossaries}
\glsdisablehyper
\newacronym{rl}{RL}{Reinforcement Learning}
\newacronym{crl}{CRL}{Contrastive Reinforcement Learning}
\newacronym{srl}{SRL}{Survival Reinforcement Learning}
\newacronym{accrl}{AC-CRL}{Action-Chunked Contrastive Reinforcement Learning}
\newacronym{tog}{ToG}{Time at Goal}
\newcommand{\ours}{ChronoSRL}
\newcommand{\ToG}{\gls{tog}}

\newcommand{\dist}{d_\theta}
\newtheorem{proposition}{Proposition}

\title{ChronoSRL: Temporal Geometry for\protect\\Self-Supervised Reinforcement Learning}
\newcommand{\website}{\href{https://nico-bohlinger.github.io/chronosrl_website}{\textbf{\textcolor{chrono}{\nolinkurl{nico-bohlinger.github.io/chronosrl_website}}}}}
\author{Nico Bohlinger\textsuperscript{1} \& Jan Peters\textsuperscript{1,2} \\
\textsuperscript{1}Technical University of Darmstadt, Germany \\
\textsuperscript{2}Robotics Institute Germany (RIG); German Research Center for AI (DFKI); hessian.AI \\
\texttt{\{nico.bohlinger,jan.peters\}@tu-darmstadt.de}}
\providecommand*{\paperversion}{arxiv}
\usepackage{etoolbox}
\newif\ifanonymous
\ifdefstring{\paperversion}{submission}{\anonymoustrue}{\iclrfinalcopy}
\begin{document}
\maketitle
\ifdefstring{\paperversion}{arxiv}{\lhead{Preprint}}{}
\ifanonymous\else\vspace*{-10pt}\centerline{\website}\vspace*{10pt}\fi

\begin{abstract}
A goal that is close in space can be far away in time. Obstacles, terrain, and the agent's own capabilities determine how long it takes to get there.
Yet, critics in contrastive and survival reinforcement learning do not measure the distances in their representation space in units of time.
We therefore introduce \ours{}, which gives the critic's embeddings an explicit temporal geometry.
The distance between state--action and goal embeddings is trained to match the time that the agent takes to reach the goal (goal-reaching time), while goals that were not reached, and goals from other trajectories, are pushed at least one discount horizon away.
Furthermore, reaching a goal quickly once does not mean that reaching it is reliable in general, so the policy should not follow the temporal distance directly.
Instead, we build on survival reinforcement learning and predict from our temporal embeddings not only the full distribution of goal-reaching times but also the time spent near the goal.
Thereby, the policy is trained to favor actions that reach the goal sooner and more reliably and that keep the agent near it.
\ours{} learns faster and reaches higher performance than contrastive, action-chunked contrastive, and survival reinforcement learning baselines on seven standard locomotion and navigation benchmarks, even with much smaller networks.
To test the limits of self-supervised reinforcement learning, we introduce velocity tracking, goal-position reaching, and box climbing tasks with a quadruped robot in a realistic sim-to-real locomotion setup, and show how the shaping terms that are typical for robotics can be naturally incorporated into our framework.
\ours{} is the only one of the tested self-supervised reinforcement learning methods that learns to stay at the commanded velocities and goal positions, and climbs the highest boxes.
\end{abstract}

\section{Introduction}
\label{sec:intro}
When an agent interacts with its environment in \gls{rl}, it collects trajectories that might or might not end up at its intended goals.
Nevertheless, each trajectory can be used as a source of information that shows how to reach the various states visited along the way.
Self-supervised \gls{rl} relabels these states as goals, turning every rollout of the agent into training data for a rich goal-conditioned policy without the need to design a hand-crafted reward for each goal \citep{andrychowicz2017hindsight,eysenbach2022contrastive}.
This makes learning possible even when the agent struggles to reach its intended goal at the start of the learning process.

For relabeled goals to be useful, the agent must learn how to relate states, and therefore goals, to one another.
Simply relying on the physical distance between states can easily lead to local optima, as often seen in maze-like environments.
The time needed to reach a particular state provides a more accurate measure of its distance, capturing the dynamics of the environment, like obstacles, and the agent's capabilities.
Naturally, every state in a trajectory provides information about how long it takes to reach any other subsequent state, simply by its time index in the trajectory.
A representation of states that encodes this temporal distance can then allow the agent to predict how long it will take to reach different goals from any given state.

We introduce \ours{} to make the temporal distance an explicit part of the representation space in self-supervised goal-conditioned \gls{rl}.
Building on temporal distance learning \citep{hartikainen2020dynamical,mendonca2021discovering,bohlinger2023intrinsically}, we train the distance between state--action and goal embeddings on the time taken to reach the goal (goal-reaching time).
Short goal-reaching times give small target distances in representation space, and long times give larger ones.
We put these targets on a shared scale by measuring time in units of the discount horizon.
For a measured goal-reaching time $\tau$ and discount factor $\gamma$, the target distance between the embeddings is $-\tau\log\gamma$, such that the objectives for representation learning and control use the same time scale.
Goals that were not reached and states from other trajectories can also be used to gain even more supervision signal, by treating them as at least one full discount horizon away.

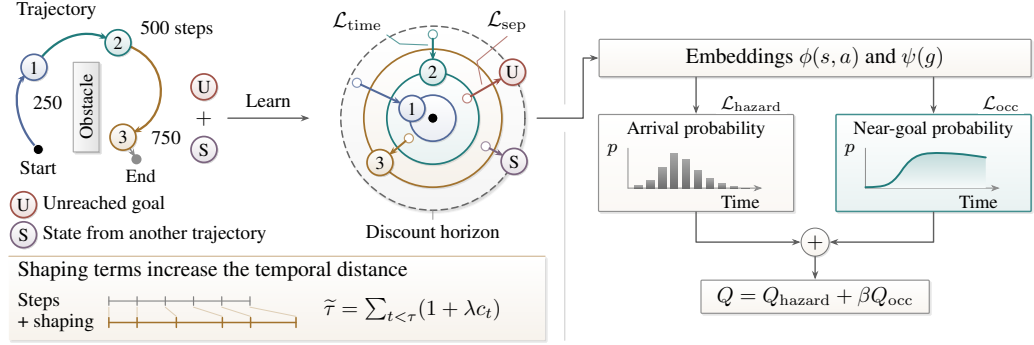
\begin{figure}[t]
\centering
\resizebox{\linewidth}{!}{%
\begingroup%
\definecolor{cBlue}{HTML}{52699A}%
\definecolor{cTeal}{HTML}{217B7B}%
\definecolor{cGold}{HTML}{A2722D}%
\definecolor{cRed}{HTML}{A05248}%
\definecolor{cViolet}{HTML}{7B647F}%
\definecolor{cGray}{HTML}{666666}%
\colorlet{cOcc}{cTeal}%
\colorlet{cShape}{cGold}%
\colorlet{cFrame}{black!45}%
\colorlet{cFlow}{black!66}%
\colorlet{cMarkerFill}{white}%
\colorlet{cOccPanel}{cOcc!3}%
\colorlet{cOccArea}{cOcc!13}%
\colorlet{cShapePanel}{cShape!4}%
\colorlet{cHazard}{black!36}%
\tikzset{%
 ringone/.style={},ringtwo/.style={},ringthree/.style={},%
 pathone/.style={},pathtwo/.style={},paththree/.style={},%
 embeddingpanel/.style={},valuepanel/.style={},%
 occarea/.style={fill=cOccArea,draw=none}%
}%
\tikzset{%
  base/.style={x=1cm,y=1cm,font=\rmfamily\fontsize{10}{11.5}\selectfont,text=black,%
    >={Stealth[length=4.1pt,width=3.0pt]},line cap=round,line join=round},%
  flow/.style={->,draw=cFlow,line width=.65pt},%
  supervision/.style={->,draw=cTeal,line width=.7pt,dash pattern=on 2.5pt off 1.5pt},%
  small/.style={font=\rmfamily\fontsize{9}{10.5}\selectfont},%
  tiny/.style={font=\rmfamily\fontsize{8.6}{9.8}\selectfont},%
  heading/.style={font=\rmfamily\fontsize{12}{13.5}\selectfont,anchor=west,inner sep=0pt},%
  label/.style={anchor=west,inner sep=0pt},%
  box/.style={draw=cFrame,line width=.55pt,fill=white,align=center,inner sep=3pt},%
  newbox/.style={box,draw=cOcc,line width=.85pt,fill=cOccPanel},%
  lossbox/.style={box,draw=cTeal,line width=.7pt,fill=cTeal!3,minimum height=.62cm},%
  point/.style={circle,draw=cBlue,fill=cMarkerFill,line width=.8pt,minimum size=12pt,inner sep=0pt,%
     font=\rmfamily\fontsize{8.5}{9.2}\selectfont},%
  origin/.style={circle,draw=white,fill=black,line width=.35pt,minimum size=4.7pt,inner sep=0pt}%
}%
\newcommand{\chronoOverviewTrajectory}[3][End of trajectory]{%
\begin{scope}[shift={(#2,#3)}]%
  \path[obstacle] (1.02,.34) rectangle (1.44,1.80);%
  \node[tiny,rotate=90] at (1.23,1.08) {Obstacle};%
  \draw[->,cBlue,pathone,line width=1.05pt] (.45,.38) .. controls (.05833,.76316) and (.02915,1.33481) .. (.24517,1.60729);%
  \draw[->,cTeal,pathtwo,line width=1.05pt,shorten >=5pt] (.38,1.72) .. controls (.74,2.27) and (1.22,2.33) .. (1.75,2.12);%
  \draw[->,cGold,paththree,line width=1.05pt,shorten >=5pt] (1.75,2.12) .. controls (2.64,1.92) and (2.60,.94) .. (1.83,.58);%
  \draw[-{Stealth[length=4.1pt,width=3pt]},black!45,line width=.65pt,shorten >=1pt] (1.83,.58)--(2.05,.22);%
  \node[origin,draw=none,fill=black!45] at (2.05,.22) {};%
  \node[origin] at (.45,.38) {};%
  \node[point,mark={cBlue}] at (.38,1.72) {1};%
  \node[point,mark={cTeal}] at (1.75,2.12) {2};%
  \node[point,mark={cGold}] at (1.83,.58) {3};%
  \node[small,anchor=west,inner sep=0pt] at (.36,1.14) {250};%
  \node[small,anchor=west,inner sep=0pt] at (2.11,2.33) {500 steps};%
  \node[small,anchor=west,inner sep=0pt] at (2.29,.58) {750};%
  \node[small,anchor=north,inner sep=1pt] at (.45,.24) {Start};%
  \node[tiny,anchor=north,inner sep=1pt] at (2.12,.10) {#1};%
\end{scope}%
}%
\newcommand{\chronoOverviewGeometry}[3]{%
\begin{scope}[shift={(#1,#2)}]%
  \def\chronoOverviewRadius{#3}%
  \chronoOverviewGeometrySurface%
  \draw[cBlue,ringone,line width=.8pt] (0,0) circle[radius=.25*\chronoOverviewRadius];%
  \draw[cTeal,ringtwo,line width=.8pt] (0,0) circle[radius=.50*\chronoOverviewRadius];%
  \draw[cGold,ringthree,line width=.8pt] (0,0) circle[radius=.75*\chronoOverviewRadius];%
  \draw[cGray,line width=.7pt,dash pattern=on 3pt off 2pt] (0,0) circle[radius=\chronoOverviewRadius];%
  \node[origin] at (0,0) {};%
  \foreach \ang/\pre/\tar/\col/\tag in {155/.91/.25/cBlue/1,90/.91/.50/cTeal/2,220/.34/.75/cGold/3,30/.43/1/cRed/U,-28/.65/1/cViolet/S}{%
     \coordinate (before) at (\ang:\pre*\chronoOverviewRadius);%
     \coordinate (after) at (\ang:\tar*\chronoOverviewRadius);%
     \coordinate (shaft-start-\tag) at ($(before)!2pt!(after)$);%
     \coordinate (shaft-end-\tag) at ($(after)!9.5pt!(before)$);%
     \coordinate (adjust-\tag) at ($(shaft-start-\tag)!.5!(shaft-end-\tag)$);%
     \node[circle,draw=\col!70,fill=white,line width=.55pt,minimum size=4.2pt,inner sep=0pt] at (before) {};%
     \draw[->,\col,line width=.95pt,shorten <=2pt,shorten >=5.4pt] (before)--(after);%
     \node[point,mark={\col}] at (after) {\tag};%
  }%
\end{scope}%
}%
\newcommand{\chronoOverviewLegend}[2]{%
  \node[point,mark={cRed}] at (#1,#2) {U};%
  \node[small,anchor=west] at ($(#1,#2)+(.19,0)$) {Unreached goal};%
  \node[point,mark={cViolet}] at ($(#1,#2)+(0,-.49)$) {S};%
  \node[small,anchor=west] at ($(#1,#2)+(.19,-.49)$) {State from another trajectory};%
}%
\newcommand{\chronoOverviewArrival}[3]{%
\begin{scope}[shift={(#1,#2)},xscale=#3]%
  \draw[->,black!45,line width=.45pt] (0,0)--(1.92,0);%
  \draw[black!45,line width=.45pt] (0,0)--(0,.65);%
  \foreach \xx/\hh in {.10/.04,.27/.13,.44/.34,.61/.56,.78/.47,.95/.29,1.12/.15,1.29/.07,1.46/.03,1.63/.01}{%
     \path[hazardbar] (\xx,0) rectangle +( .105,\hh);%
  }%
\end{scope}%
}%
\newcommand{\chronoOverviewOccupancy}[3]{%
\begin{scope}[shift={(#1,#2)},xscale=#3]%
  \draw[->,black!45,line width=.45pt] (0,0)--(1.92,0);%
  \draw[black!45,line width=.45pt] (0,0)--(0,.65);%
  \path[occarea] (0,0)--(0,.015)..controls(.28,.02) and (.36,.03)..(.51,.27)%
    ..controls(.68,.57) and (.83,.58)..(1.12,.57)..controls(1.42,.56)and(1.53,.53)..(1.72,.50)%
    --(1.72,0)--cycle;%
  \draw[cOcc,line width=.9pt] (0,.015)..controls(.28,.02) and (.36,.03)..(.51,.27)%
    ..controls(.68,.57) and (.83,.58)..(1.12,.57)..controls(1.42,.56)and(1.53,.53)..(1.72,.50);%
\end{scope}%
}%
\tikzset{%
 mark/.style={draw=#1},%
 obstacle/.style={fill=black!5,draw=black!28,line width=.5pt},%
 hazardbar/.style={fill=cHazard,draw=none},%
 hazardpanel/.style={},%
 sumnode/.style={}%
}%
\newcommand{\chronoOverviewGeometrySurface}{}%
\newcommand{\chronoOverviewShapingSurface}{%
 \fill[cShapePanel] (0,.35) rectangle (8.74,1.72);%
}%
\definecolor{cBlue}{HTML}{52699A}%
\definecolor{cTeal}{HTML}{217B7B}%
\definecolor{cGold}{HTML}{A2722D}%
\definecolor{cRed}{HTML}{A05248}%
\definecolor{cViolet}{HTML}{7B647F}%
\definecolor{cGray}{HTML}{666666}%
\colorlet{cOcc}{cTeal}%
\colorlet{cShape}{cGold}%
\colorlet{cOccPanel}{cOcc!3}%
\colorlet{cOccArea}{cOcc!13}%
\colorlet{cShapePanel}{cShape!4}%
\colorlet{cFrame}{black!43}%
\colorlet{cFlow}{black!63}%
\tikzset{%
 box/.append style={rounded corners=.5pt,line width=.55pt,%
   top color=white,bottom color=cGold!5,%
   blur shadow={shadow xshift=.8pt,shadow yshift=-.8pt,%
     shadow blur radius=.55pt,shadow opacity=17,shadow blur steps=8}},%
 newbox/.append style={draw=cOcc!85,line width=.7pt,%
   top color=white,bottom color=cOcc!10},%
 embeddingpanel/.style={top color=white,bottom color=cGold!7},%
 valuepanel/.style={top color=white,bottom color=cGold!7},%
 mark/.style={draw=#1,top color=white,bottom color=#1!20,%
   blur shadow={shadow xshift=.45pt,shadow yshift=-.5pt,%
     shadow blur radius=.25pt,shadow opacity=27,shadow blur steps=8}},%
 obstacle/.style={draw=black!32,line width=.55pt,%
   left color=black!2,right color=black!11,%
   blur shadow={shadow xshift=.9pt,shadow yshift=-.85pt,%
     shadow blur radius=.25pt,shadow opacity=25,shadow blur steps=6}},%
 hazardbar/.style={draw=none,top color=black!35,bottom color=black!57},%
 occarea/.style={draw=none,top color=cOcc!28,bottom color=cOcc!4},%
 sumnode/.style={top color=white,bottom color=cGold!8,%
   blur shadow={shadow xshift=.5pt,shadow yshift=-.5pt,%
     shadow blur radius=.35pt,shadow opacity=21,shadow blur steps=6}}%
}%
\renewcommand{\chronoOverviewGeometrySurface}{%
 \path[top color=white,bottom color=black!4,%
   blur shadow={shadow xshift=.8pt,shadow yshift=-.9pt,%
     shadow blur radius=1.1pt,shadow opacity=21,shadow blur steps=10}]%
   (0,0) circle[radius=\chronoOverviewRadius];%
 \path[top color=white,bottom color=cGold!10,%
   blur shadow={shadow xshift=.55pt,shadow yshift=-.65pt,%
     shadow blur radius=.6pt,shadow opacity=19,shadow blur steps=8}]%
   (0,0) circle[radius=.75*\chronoOverviewRadius];%
 \path[top color=white,bottom color=cTeal!10,%
   blur shadow={shadow xshift=.5pt,shadow yshift=-.6pt,%
     shadow blur radius=.6pt,shadow opacity=19,shadow blur steps=8}]%
   (0,0) circle[radius=.50*\chronoOverviewRadius];%
 \path[top color=white,bottom color=cBlue!12,%
   blur shadow={shadow xshift=.45pt,shadow yshift=-.55pt,%
     shadow blur radius=.5pt,shadow opacity=19,shadow blur steps=8}]%
   (0,0) circle[radius=.25*\chronoOverviewRadius];%
}%
\renewcommand{\chronoOverviewShapingSurface}{%
 \path[rounded corners=.5pt,top color=white,bottom color=cShape!9,%
   blur shadow={shadow xshift=.8pt,shadow yshift=-.8pt,%
     shadow blur radius=.55pt,shadow opacity=17,shadow blur steps=8}]%
   (0,.35) rectangle (8.74,1.72);%
}%
\begin{tikzpicture}[base]
\path[use as bounding box] (-.03,.31) rectangle (17.18,5.94);
\chronoOverviewTrajectory[End]{.03}{3.08}
\node[small,anchor=west] at (.00,5.73) {Trajectory};
\node[point,mark={cRed}] at (3.18,4.48) {U};
\node[point,mark={cViolet}] at (3.18,3.46) {S};
\node at (3.18,3.97) {$+$};
\draw[flow] (3.57,3.97)--(4.90,3.97);
\node[small] at (4.23,4.25) {Learn};
\chronoOverviewGeometry{6.91}{3.97}{1.50}
\node[inner sep=2pt] (ltime) at (5.68,5.55) {$\mathcal L_{\rm time}$};
\draw[cTeal!75,line width=.5pt,shorten >=2.4pt] (ltime.south east)--(6.37,5.18)--(adjust-2);
\node[inner sep=2pt] (lsep) at (8.12,5.55) {$\mathcal L_{\rm sep}$};
\draw[cRed!75,line width=.5pt,shorten >=2.4pt] (lsep.south)--(7.63,4.90)--(adjust-U);
\node[small] at (6.91,2.13) {Discount horizon};
\draw[black!40,line width=.45pt] (6.91,2.32)--(6.91,2.46);
\chronoOverviewLegend{.24}{2.56}
\draw[black!19,line width=.5pt] (9.06,.35)--(9.06,3.74);
\draw[black!19,line width=.5pt] (9.06,4.20)--(9.06,5.72);
\node[box,embeddingpanel,minimum width=7.04cm,minimum height=.61cm] (emb) at (13.14,4.95)
 {Embeddings $\phi(s,a)$ and $\psi(g)$};
\draw[flow] (8.61,3.97)--(9.36,3.97)--(9.36,4.95)--(emb.west);
\node[box,hazardpanel,minimum width=3.18cm,minimum height=1.58cm] (arr) at (11.20,3.24) {};
\node[tiny,anchor=south east,fill=white,inner xsep=2pt,inner ysep=1pt]
 at ($(arr.north east)+(-.12,0)$) {$\mathcal L_{\rm hazard}$};
\chronoOverviewArrival{10.09}{2.83}{1.14}
\node[tiny,anchor=east] at (12.36,2.61) {Time};
\node[tiny,anchor=east] at (10.04,3.45) {$p$};
\node[tiny] at (11.20,3.79) {Arrival probability};
\node[newbox,minimum width=3.18cm,minimum height=1.58cm] (occ) at (15.07,3.24) {};
\node[tiny,anchor=south east,fill=white,inner xsep=2pt,inner ysep=1pt]
 at ($(occ.north east)+(-.12,0)$) {$\mathcal L_{\rm occ}$};
\chronoOverviewOccupancy{13.96}{2.83}{1.14}
\node[tiny,anchor=east] at (16.24,2.61) {Time};
\node[tiny,anchor=east] at (13.91,3.45) {$p$};
\node[tiny] at (15.07,3.79) {Near-goal probability};
\draw[flow] (emb.south -| arr.north)--(arr.north);
\draw[flow] (emb.south -| occ.north)--(occ.north);
\node[circle,draw=black!55,fill=white,sumnode,minimum size=13pt,inner sep=0pt] (sum) at (13.14,1.95) {$+$};
\draw[flow] (arr.south)--(11.20,1.95)--(sum.west);
\draw[flow] (occ.south)--(15.07,1.95)--(sum.east);
\node[box,valuepanel,minimum width=3.72cm,minimum height=.54cm] (pol) at (13.14,1.06)
  {$Q=Q_{\rm hazard}+\beta Q_{\rm occ}$};
\draw[flow] (sum.south)--(pol.north);
\chronoOverviewShapingSurface
\draw[cShape!55,line width=.5pt] (0,1.72)--(8.74,1.72);
\node[anchor=west,inner sep=0pt] at (.14,1.48) {Shaping terms increase the temporal distance};
\node[small,anchor=west,inner sep=0pt] at (.14,.99) {Steps};
\draw[black!45,line width=.55pt] (1.63,.99)--(3.93,.99);
\foreach \x in {1.63,2.09,2.55,3.01,3.47,3.93}{\draw[black!50,line width=.55pt](\x,.92)--(\x,1.06);}
\node[small,anchor=west,inner sep=0pt] at (.14,.64) {+ shaping};
\draw[cShape,line width=.70pt] (1.63,.64)--(4.68,.64);
\foreach \x in {1.63,2.09,2.73,3.48,3.94,4.68}{\draw[cShape,line width=.7pt](\x,.56)--(\x,.72);}
\foreach \u/\v in {1.63/1.63,2.09/2.09,2.55/2.73,3.01/3.48,3.47/3.94,3.93/4.68}{
  \draw[cShape!35,line width=.4pt] (\u,.89)--(\v,.74);
}
\node[align=center] at (6.58,.84) {$\widetilde\tau=\textstyle\sum_{t<\tau}(1+\lambda c_t)$};
\end{tikzpicture}%
\endgroup%
}
\caption{
\ours{} learns temporal geometry from trajectories.
State 3 is physically closer to the start than state 2 but takes longer to reach.
The time loss trains embedding distances to match goal-reaching times, while the separation loss pushes unreached goals and states from other trajectories at least one discount horizon away.
The shared embeddings also predict the distribution over goal-reaching times with a hazard head and the time spent near the goal with an occupancy head.
Both heads form the value that trains the policy, and shaping terms can add per-step costs to the distance.
}
\label{fig:overview}
\end{figure}

Giving the distance between embeddings a temporal meaning does not necessarily make them sufficient for choosing actions.
A risky action can reach a goal quickly on a single successful attempt and still fail most of the time.
An agent can also learn to overshoot and reach a goal region briefly and fail to stay there consistently.
Therefore, learning must account for both the uncertainty in reaching the goal and what happens after reaching it.
Thus, we build on the idea of \gls{srl} \citep{nguimatsia2026survival}, and use our temporal embeddings to predict both the full distribution over goal-reaching times and the time spent near the goal.
The policy is trained on the value obtained from these predictions, and learns to favor actions that reach the goal sooner and more reliably and keep the agent nearby.

We evaluate \ours{} on seven standard locomotion and navigation environments \citep{bortkiewicz2024jaxgcrl} and compare it to \gls{crl} \citep{wang2025scaling}, \gls{accrl} \citep{korniak2026chunking}, and \gls{srl} \citep{nguimatsia2026survival}.
\ours{} learns faster than all three baselines across network depths and typically reaches the final performance of the strongest baseline with less than a third of the training budget.
Already with eight layers, it learns faster than any baseline with up to $64$ layers.
Beyond these standard benchmarks, we evaluate the limits of \ours{} and the baseline methods in more challenging and realistic robotics scenarios.
We build on the typical sim-to-real pipeline for legged robot locomotion and design velocity tracking, goal-position reaching, and box climbing tasks for the Unitree Go2 quadruped robot.
We show how the shaping terms that are necessary and typical for robotics can be naturally incorporated into the self-supervised \gls{rl} framework.
\ours{} is the only one of these methods that learns to stay at the commanded velocities and goal positions, and climbs the highest boxes.

\section{Related work}
\label{sec:related}

Goal-conditioned \gls{rl} trains a single policy to reach many different goals \citep{liu2022goal}.
\citet{kaelbling1993learning} already showed how each observed transition of an agent could be used to update its estimates of the expected remaining steps for multiple goals.
Afterwards, universal value function approximators were introduced, generalizing over multiple goals through a shared function approximator \citep{schaul2015universal}.
Such goal-conditioned value functions can also learn from trajectories that do not reach their intended goal, since they still reach other states.
Hindsight experience replay uses these achieved states as alternative goals for learning the value function \citep{andrychowicz2017hindsight}, while the relabeled trajectories can also be used as demonstrations for training the policy directly with supervised learning \citep{ghosh2021learning}.

The same trajectories can also be used to learn representations that predict which states the agent is likely to visit.
Successor representations describe these predictions through discounted future visitations \citep{dayan1993improving,barreto2017successor}.
One way to learn such predictions is to distinguish future states from randomly sampled goals, following the ideas of noise-contrastive estimation and contrastive predictive coding \citep{gutmann2010noise,oord2018representation}.
C-learning and \gls{crl} use this classification problem to learn goal-conditioned value functions \citep{eysenbach2021clearning,eysenbach2022contrastive}.
Temporal-difference updates extend this approach by combining information across trajectories \citep{zheng2024contrastive}, while improvements to network architectures and training make \gls{crl} effective on offline robotics data \citep{zheng2024stabilizing}.
Alongside recent progress in scaling value functions in \gls{rl} \citep{nauman2024bigger,lee2025simba,palenicek2026xqc}, \citet{wang2025scaling} show that increasing network depth allows \gls{crl} to continuously scale and improve its performance.
Building on the idea of action chunking \citep{zhao2023learning,li2026reinforcement}, \citet{korniak2026chunking} extend \gls{crl} to action sequences and improve performance even further.

While \gls{crl} uses InfoNCE to distinguish future goals from randomly sampled goals \citep{oord2018representation,eysenbach2022contrastive,wang2025scaling}, the same trajectories also provide step counts between states, which temporal-distance methods use to learn distances and construct rewards for goal reaching \citep{florensa2019self,venkattaramanujam2019selfsupervised,hartikainen2020dynamical,mendonca2021discovering,bohlinger2023intrinsically}.
Each state can be paired with later states, providing targets over many distances within a trajectory.
Such distances can depend on the direction, since reaching a state may take longer than returning from it.
Quasimetric models capture this asymmetry while satisfying the triangle inequality \citep{wang2022quasimetric,liu2023metric}.
Quasimetric \gls{rl} builds on these models with an objective for learning optimal goal-reaching distances \citep{wang2023optimal,zheng2026multistep}, and \citet{myers2024learning} show that a temporal quasimetric can also be derived from contrastive successor representations.
In offline and unsupervised goal-conditioned \gls{rl}, temporal distances can be used to structure the latent space \citep{park2024foundation,park2024metra,bae2024tldr,myers2026offline,boock2026hitting} and guide the selection of subgoals \citep{park2023hiql}.
\ours{} introduces its temporal structure through a learning objective rather than enforcing a quasimetric distance function.
We train distances between state--action and goal embeddings to match observed goal-reaching times in units of the discount horizon.
Whereas other methods use the learned distance itself as the value or reward that the policy maximizes, \ours{} uses it only to shape the representation of its critic, whose value estimate is then a combination of the predictions of the goal-reaching time and goal-region occupancy.

Goal-conditioned value functions can also be learned by predicting when a goal will be reached.
A trajectory that ends at the intended goal provides a goal-reaching time, while an attempt that ends before reaching it provides a lower bound on this time.
Survival analysis uses these measured times and lower bounds to learn a distribution over possible goal-reaching times \citep{kaplan1958nonparametric,cox1972regression}.
Building on neural survival models \citep{katzman2018deepsurv,kvamme2019time,gensheimer2019scalable,lee2018deephit}, \gls{srl} predicts this distribution for each state--action--goal triplet and derives a discounted value from it \citep{nguimatsiatiofack2026svl,nguimatsia2026survival}.
Since reaching a goal once does not ensure that the agent stays there, \citet{nguimatsia2026survival} introduce a dwell-time objective that requires the agent to remain at the goal for a consecutive period of time.
\ours{} builds on this idea by directly predicting the goal occupancy, which is the expected time spent near the goal.
However, only reaching and remaining at a goal does not say anything about how the agent got there.
When applying self-supervised \gls{rl} algorithms to robotics tasks that should be transferable to the real world, it is important to ensure that the learned policies produce smooth trajectories or are influenced by other relevant constraints, such as energy consumption.
Such reward shaping terms help the real-world transfer and can also improve the exploration and learning efficiency of the agent \citep{pmlr-v164-rudin22a,bohlinger2024onepolicy}.
We incorporate these terms as per-step costs in the temporal targets, such that both the learned geometry and the arrival and occupancy predictions account for them.

\section{Temporal geometry for self-supervised RL}
\label{sec:method}
We propose to give the critic in self-supervised \gls{rl} a representation in which distances measure time.
The core idea is shown in Figure~\ref{fig:overview}.
Instead of learning the critic's embeddings only as features for its value prediction, we train the distance between state--action and goal embeddings to match the goal-reaching times observed in the agent's own trajectories.
Building on \gls{srl} \citep{nguimatsia2026survival}, the same embeddings then predict when the goal will be reached and how much time the agent will spend near it.
The policy is trained on the value combined from both these predictions.

\subsection{Building on survival reinforcement learning}
\label{sec:temporal_supervision}
We consider goal-conditioned \gls{rl} with states $s$, actions $a$, goals $g$, and a discount factor $\gamma$, where the agent reaches a goal once the goal-relevant part of its state, such as its position, lies within a distance $\varepsilon$ of $g$, which defines the goal region $\mathcal B_\varepsilon(g)$.
Following \gls{accrl}, the policy $\pi(a\mid s,g)$ outputs a chunk of actions that is executed open-loop, and we write $a$ for such a chunk, while all times count individual environment steps.
We use chunks of two actions, as they worked best for \ours{}, compared to the three actions of \gls{accrl}.
The critic learns from random batches of training examples, drawn from the trajectories that the agent collects with this policy.
Each training example pairs a stored state $s$ and a chunk $a$ with a goal $g$, which hindsight relabeling can take from a state that the agent reached later in the same trajectory.
The time indices of the trajectory from $s$ onward then provide the goal-reaching time $\tau$, which is the number of steps until it first enters $\mathcal B_\varepsilon(g)$, with $\tau=0$ if the goal is already reached.
If the goal is not reached within an observation window of $W$ steps, we only know that $\tau\ge W$, and survival analysis calls such an example censored.

We build on \gls{srl}, which learns a goal-conditioned value from all of these examples.
Its critic encodes the state and the action chunk as an embedding $\phi_\theta(s,a)$ and the goal as an embedding $\psi_\theta(g)$, where $\theta$ denotes all parameters of the critic.
Based on both embeddings, a hazard head $h_\theta(t\mid s,a,g)$ predicts the probability of reaching the goal at step $t$, given that it has not been reached before.
The survival probability that the goal has still not been reached at step $t$ is then $S_\theta(t\mid s,a,g)=\prod_{u=0}^{t}\bigl(1-h_\theta(u\mid s,a,g)\bigr)$.
The hazard head is trained with the negative log-likelihood $\mathcal L_{\mathrm{hazard}}$ of all examples, including the censored ones (see Appendix~\ref{app:survival}).
If episodes ended at the goal, the value of an action would be the negative discounted time spent waiting for it,
\begin{equation}
Q_{\mathrm{hazard}}(s,a,g)=-\sum_{t=0}^{W-1}\gamma^t\,S_\theta(t\mid s,a,g).
\label{eq:arrival}
\end{equation}
Since, in the environments we consider, episodes do not end immediately upon reaching the goal, \gls{srl} relabels sequences of $k$ consecutive goals, such that reaching the goal requires staying near it for $k$ steps.
In this critic, nothing relates the distance between the two embeddings to the time it takes to reach the goal, although this time is exactly what the critic predicts.

\subsection{Distances in units of time}
\label{sec:temporal_geometry}
We measure the distance between the two embeddings as $\dist(s,a,g)=\|\phi_\theta(s,a)-\psi_\theta(g)\|_2$.
A goal reached after $\tau$ steps is discounted by $\gamma^\tau=e^{-\kappa\tau}$ with $\kappa=-\log\gamma$.
We use the negative log of this discount, $\kappa\tau$, as the target distance, such that a goal at distance $d$ is discounted by exactly $e^{-d}$.
The distance thereby measures time in units of the discount horizon $1/\kappa\approx1/(1-\gamma)$, which is about $1000$ steps for $\gamma=0.999$.
We fit the distances to all measured goal-reaching times with the time loss
\begin{equation}
\mathcal L_{\mathrm{time}}=\mathbb E_{(s,a,g,\tau)\sim\mathcal D_+}\Big[\rho\big(\dist(s,a,g)-\kappa\tau\big)\Big],
\label{eq:time}
\end{equation}
where $\rho$ is the Huber loss with a threshold of one discount horizon and $\mathcal D_+$ contains the examples with a measured $\tau\ge1$, since goals that are already reached are left to the hazard head.

The time loss only constrains pairs that a trajectory connects, so the encoders could still place unrelated states and goals close to each other.
Therefore, we also use the pairs without a measured goal-reaching time, marked U and S in Figure~\ref{fig:overview}, which are the censored examples and the cross-batch pairs of each state--action with the goals of all other examples in the batch, excluding goals within $\varepsilon$ of its own goal.
For these pairs, we only know, or assume, that reaching the goal takes at least $W$ steps, so we penalize only distances below $\kappa W$, with a linear version of the margin loss of contrastive metric learning \citep{hadsell2006dimensionality}, which we call the separation loss,
\begin{equation}
\mathcal L_{\mathrm{sep}}=\frac{1}{|\mathcal N|}\sum_{(i,j)\in\mathcal N}\big[\kappa W-\dist(s_i,a_i,g_j)\big]_+,
\label{eq:separation}
\end{equation}
where $[x]_+=\max(x,0)$, $\mathcal N$ contains the censored examples and the cross-batch pairs of a batch, and with $\gamma=0.999$ and $W=1000$, the margin is almost exactly one discount horizon.
Unlike the censored examples, the cross-batch pairs never enter $\mathcal L_{\mathrm{hazard}}$, since no trajectory shows that their goals are at least $W$ steps away, and some may in fact be closer (see Appendix~\ref{app:geometry_losses}).

\subsection{From temporal geometry to policy learning}
\label{sec:readout}
It is tempting to simply convert the distance to a goal into its discount, $e^{-d}$, and use it as the value that the policy maximizes.
However, the time loss only sees reached goals, so two actions that reach the same goal after the same number of steps have the same target distance, even if one succeeds far more often than the other (see Appendix~\ref{app:waiting_value} for a formal argument).
The distance also ignores what happens after the goal is reached.
Therefore, we use the temporal embeddings to predict when the goal is reached and how long the agent stays near it.

For the first prediction, we keep the hazard head of \gls{srl} on top of the temporal embeddings.
Since it is also trained on the censored examples, it captures how likely the goal is reached at all.
It receives both embeddings as input rather than only their distance, so pairs at a similar temporal distance can still differ in how reliably their goal is reached.
For the second prediction, we add an occupancy head, which predicts the probability $O_t=O_\theta(t\mid s,a,g)$ that the agent is in the goal region $t$ steps later, and train it on the occupancies $o_t$ observed in the rollout trajectories.
After an environment termination, for example a fall, the observed occupancy is zero for the rest of the window.
To reward being close to the center of the goal region, we average the occupancy over the goal regions with radii $\varepsilon$, $\varepsilon/2$, and $\varepsilon/4$ (see Appendix~\ref{app:occupancy}).
The occupancy head's loss and discounted value are
\begin{equation}
\mathcal L_{\mathrm{occ}}=-\mathbb E\Big[\sum_t o_t\log O_t+(1-o_t)\log(1-O_t)\Big]
\quad\text{and}\quad
Q_{\mathrm{occ}}(s,a,g)=\sum_{t=1}^{W-1}\gamma^t\,O_t.
\label{eq:occupancy_value}
\end{equation}
Up to normalization, $Q_{\mathrm{occ}}$ is the discounted goal occupancy that \gls{crl} learns as its value \citep{eysenbach2022contrastive}, but predicted per future step and fit directly to the observed occupancies.
Unlike $Q_{\mathrm{hazard}}$, which only depends on the first arrival, it distinguishes staying at the goal from passing through it or falling after arriving.

Both value estimates, $Q_{\mathrm{hazard}}$ and $Q_{\mathrm{occ}}$, count discounted time steps, the steps spent waiting for the goal and the steps spent near it, so the policy maximizes their weighted sum with an entropy bonus, as in Soft Actor-Critic \citep{haarnoja2018soft},
\begin{equation}
J(\pi)=\mathbb E_{(s,g)\sim\mathcal D,\,a\sim\pi(\cdot\mid s,g)}\big[Q_{\mathrm{hazard}}(s,a,g)+\beta\,Q_{\mathrm{occ}}(s,a,g)-\alpha\log\pi(a\mid s,g)\big],
\label{eq:actor}
\end{equation}
where $\mathcal D$ is the replay buffer and $\alpha$ the entropy coefficient.
The weight $\beta$ is the fraction of examples in the batch whose episode terminates within the window, so the policy follows the occupancy value less when the agent rarely falls and more when it falls more often (see Appendix~\ref{app:actor_details}).

Finally, the critic minimizes the sum $\mathcal L_{\mathrm{hazard}}+\mathcal L_{\mathrm{occ}}+\mathcal L_{\mathrm{time}}+\mathcal L_{\mathrm{sep}}$ with equal coefficients, and all four losses update the same encoders (see Algorithm~\ref{alg:chronosrl}).
The time and separation losses, which we call the geometric losses, add no parameters, and without them, our critic reduces to a survival critic as in \citet{nguimatsia2026survival}, with an occupancy head and action chunks.

\subsection{Shaping terms as additional time costs}
\label{sec:shaping}
Policies that control robots need to act smoothly and reach their goals in a way that humans prefer, which typical reward-based \gls{rl} encodes as reward shaping terms \citep{pmlr-v164-rudin22a,bohlinger2024onepolicy}.
Self-supervised \gls{rl} has no reward terms, but all targets of our critic count time steps, so we simply express these terms as additional time costs.
For a per-step cost $c_t\ge0$ and a weight $\lambda\ge0$, we replace the time $\tau$ in our temporal targets, including the censoring, by the shaped time $\widetilde\tau=\tau+\lambda\sum_{t<\tau}c_t$, so every step with a nonzero cost counts as more than one step.
Two arrivals after the same number of steps can then differ in distance and value if one accumulates more cost.
The policy objective stays unchanged, and $\lambda=0$ recovers the unshaped step counts (see Appendix~\ref{app:shaping}).

\section{Experiments}
\label{sec:experiments}
We evaluate \ours{} on the seven locomotion and navigation tasks that \citet{wang2025scaling} use in JaxGCRL \citep{bortkiewicz2024jaxgcrl}.
Four of these tasks are ant mazes of different sizes and layouts (U4-Maze, Big Maze, Hardest Maze, and U5-Maze), and in the other three, a humanoid has to walk to its goal, either in the open or through a maze (Humanoid, Humanoid U-Maze, and Humanoid Big Maze).
Our baselines are the self-supervised methods that have been scaled to these tasks.
We run \gls{crl} \citep{wang2025scaling}, \gls{accrl} \citep{korniak2026chunking}, and \gls{srl} \citep{nguimatsia2026survival} with their published hyperparameters.
Since network depth is the main axis along which these methods scale, we train all four methods at depths from $1$ to $64$, with four seeds per setting and $100$ million environment steps per run.
We note that this budget is smaller than the budgets that \citet{wang2025scaling} use for the three hardest mazes (see Appendix~\ref{app:reference_methods}).
We measure performance by the \ToG{}, which counts the steps that the agent spends in the goal region during an evaluation episode of $1000$ steps.
We summarize each learning curve by its final value and by its area (AUC), normalized by the budget (see Appendix~\ref{app:provenance}).

\subsection{\ours{} learns faster with smaller networks}
\label{sec:performance}
\begin{figure}[t]
\centering
\includegraphics[width=\linewidth]{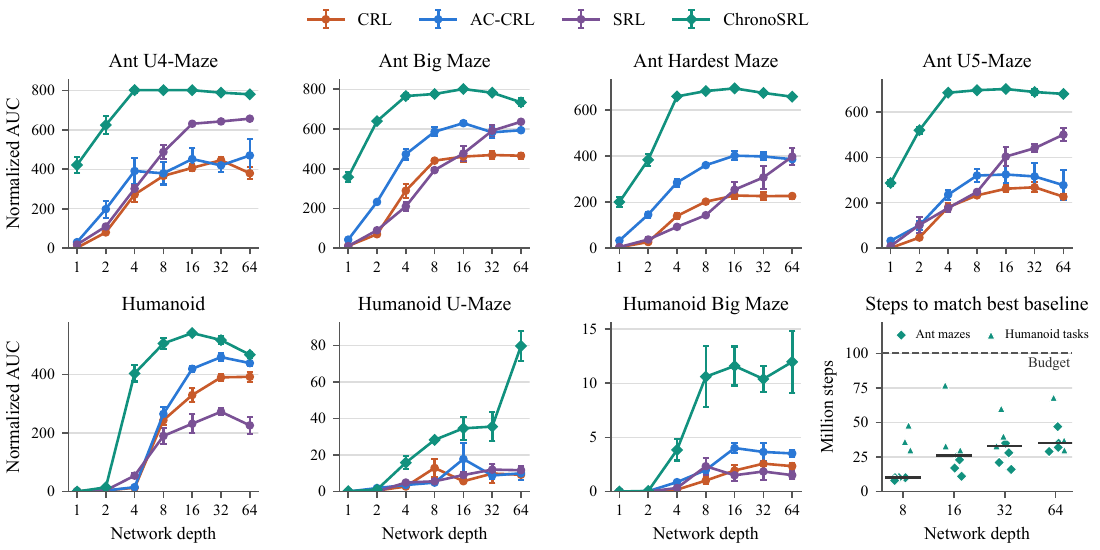}
\caption{\ours{} learns faster and achieves higher final performance than all baselines across network depths.
The first seven panels show the mean normalized AUC over four seeds with standard errors, and the last panel shows the steps that \ours{} needs to reach the final \ToG{} of the strongest baseline at the same depth.}
\label{fig:depth_auc}
\end{figure}
Figure~\ref{fig:depth_auc} shows that \ours{} learns much faster than all baselines, with the highest AUC on every task and at every depth, except for a few humanoid settings with one or two layers in which no method learns.
With eight or more layers, it needs a median of only $30$ million steps, less than a third of the budget, to reach the final performance of the strongest baseline (Figure~\ref{fig:depth_auc}, last panel).
We chose to report the AUC because it highlights this speed, but \ours{} also has the highest mean final \ToG{} in all but one of these settings, with slightly smaller margins (see Figure~\ref{fig:depth_final} in Appendix~\ref{app:full_results}).
\ours{} also needs much smaller networks, as in all four ant mazes, a critic with only four layers learns faster than every baseline at any depth up to $64$.
With the same number of environment steps, such a run takes only about a fifth of the wall-clock time of a baseline with $64$ layers on the same GPUs (see Appendix~\ref{app:reference_methods} for the run times).
In the humanoid tasks, \ours{} needs eight layers for this, and only Humanoid U-Maze still clearly benefits from $64$ layers.

\subsection{What makes \ours{} work}
\label{sec:ablations}
\begin{figure}[t]
\centering
\includegraphics[width=\linewidth]{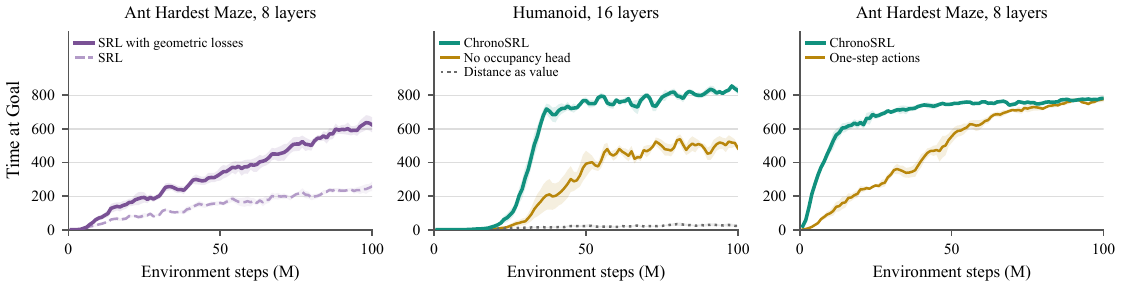}
\caption{Added to \gls{srl}, the geometric losses more than double its AUC on Ant Hardest Maze, and removing the occupancy head or the action chunks slows down \ours{}.
Without the occupancy head, the humanoid falls after reaching its goal, and a policy that follows the distance does not learn.
All lines show the smoothed mean \ToG{} over four seeds with one standard error.}
\label{fig:ablations}
\end{figure}
To see where these gains come from, we study the components of \ours{} (see Figure~\ref{fig:ablations} and Appendix~\ref{app:ablation_protocol}).
The geometric losses are not specific to our critic, as adding them to \gls{srl} speeds it up in all four ant mazes and more than doubles its AUC on Ant Hardest Maze.
Within \ours{}, they mainly speed up early learning on the hardest tasks.
The occupancy head is the major component that keeps the agent near its goal.
Without it, the humanoid often falls after reaching the goal, as it overshoots with too much momentum, and loses about $40\%$ of its final \ToG{}.
A policy that simply follows the distance instead of the predicted value does not learn at all, as the distance tells the agent neither how reliably it reaches the goal nor whether it stays there.
Finally, the action chunks significantly speed up the start of learning process, as with single actions, \ours{} needs about three times as many steps to reach $80\%$ of its final \ToG{} on Ant U4-Maze and Ant Hardest Maze.

\subsection{Inside the temporal geometry}
\label{sec:geometry}
\begin{figure}[t]
\centering
\includegraphics[width=\linewidth]{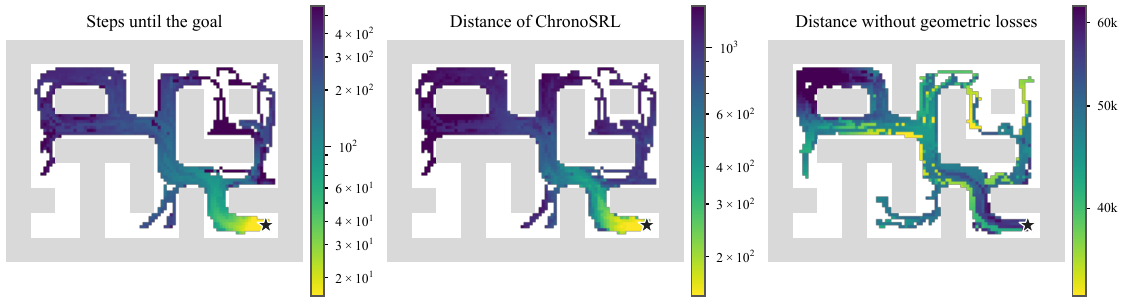}
\caption{The distance of \ours{} follows the time that the agent needs to reach the goal.
For a goal (star) in the far right corner of Ant Hardest Maze with $8$ layers, the maps show the steps that the final policies still need from each position, the critic's distance to the goal converted to steps, and the same distance without the geometric losses, pooled over four seeds, each on its own color scale.}
\label{fig:geometry}
\end{figure}
In mazes, the distance between two states in space is a poor proxy to the actual time that reaching a goal takes.
To see whether the critic of \ours{} measures the goal-reaching time properly, we start the ant at every free cell of a maze, let the final policy walk to the same goal, and, along the episodes that reach it, convert the critic's distance to the goal into steps by dividing it by $\kappa$ (see Figure~\ref{fig:geometry}).
In Ant Hardest Maze, the distance of \ours{} follows the steps that the agent actually needs, with a strong rank correlation of $0.93$, while without the geometric losses, the distances are at tens of thousands of steps and, across the maze, are not even clearly positively related to the remaining steps.
The difference between the distances in space and time is the largest in Ant U4-Maze, where the goal lies behind a wall, so that positions close to it in space can be far from it in time, like the state 3 in Figure~\ref{fig:overview}.
There, the distance of \ours{} follows the steps that the agent still needs, while without the geometric losses, it partly follows the straight-line distance instead (see Appendix~\ref{app:measurements} for more details).
Both geometric losses are needed for this, as without the time loss, the distances are unrelated to the remaining steps, and without the separation loss, they follow the straight-line distance instead.
The separation loss also stretches the scale of the distances, which can lead to overestimations of the remaining steps (see Appendix~\ref{app:measurements} for the maps of all variants).
For control, this overestimate does not matter, as the policy does not act on the distance but on the predicted hazard and occupancy values, which matches the discounted waiting time that the agent actually experiences along evaluation episodes, with a rank correlation of $0.99$ in both ant mazes.

\subsection{Getting to the goal and staying there}
\label{sec:behavior}
\begin{figure}[t]
\centering
\includegraphics[width=\linewidth]{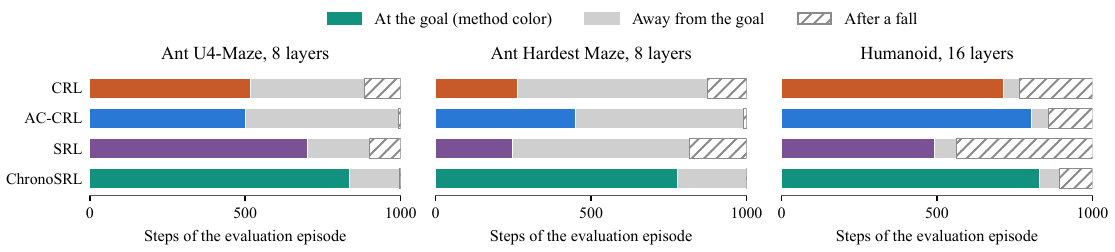}
\caption{\ours{} spends most of each evaluation episode at the goal.
Bars split the $1000$ steps into the time at the goal (colored), away from it (gray), and after a fall (hatched) for the final policies.}
\label{fig:behavior}
\end{figure}
To see how \ours{} achieves its higher \ToG{}, we split each evaluation episode into the time at the goal, the time away from it, and the time lost after a fall.
Figure~\ref{fig:behavior} shows that in all four ant mazes, \ours{} with eight layers reaches the goal in more than $97\%$ of the evaluation episodes, against $34$ to $92\%$ for the baselines (see Table~\ref{tab:reach} in Appendix~\ref{app:full_results} for more details), and then stays there, while the baselines spend far more time away from the goal, and \gls{srl} and \gls{crl} regularly flip over and fall.
In the Humanoid environment, every method reaches the goal in more than $90\%$ of the episodes, but afterwards \gls{srl} loses more than $40\%$ of the episode because the humanoid falls, while \ours{} keeps the humanoid upright near the goal, which is exactly what our occupancy head encourages.

\subsection{Toward self-supervised robot locomotion}
\label{sec:robotics}
Beyond the standard benchmarks, we evaluate self-supervised \gls{rl} on more realistic robotic tasks and study whether these methods can learn the behaviors that legged robots need for real world deployment, with three tasks for the Unitree Go2 quadruped, implemented in the RL-X framework \citep{bohlinger2023rlx}, including the necessary domain randomization, shaping terms, and learning curricula.
In velocity tracking, the goal is a command for the forward, lateral, and yaw velocity of the robot's base, in goal-position reaching, it is a position of up to five meters away together with a heading angle, and in box climbing, the goal is on top and in the center of a box whose height grows with a performance-based curriculum \citep{magnus2025investigating}.
All methods add their actions to a trot-shaped gait prior, which we found necessary to speed up learning,
while the rest of the action and observation space follow the standard pipeline in legged robot locomotion \citep{pmlr-v164-rudin22a}.
We scale up training to $4096$ parallel environments for $500$ million steps, and both survival critics, \gls{srl} and \ours{}, receive the same shaping terms of the pipeline as additional time (as described in Section~\ref{sec:shaping}), which the contrastive critics cannot use.

Figure~\ref{fig:robotics} shows that, from four layers on, \ours{} is the only method that learns to stay at its goal for most of the episode in velocity tracking and goal-position reaching, and at the end of training, it tracks the forward command with an error of about $0.1$ to $0.2$ meters per second, against $0.3$ to $0.6$ for the baselines (more details in Appendix~\ref{app:robots}).
In velocity tracking, the episodes of \gls{crl} and \gls{accrl} end within about $40$ steps, as they just learn to fall into the direction of the commanded velocity.
On box climbing, from eight layers on, \ours{} gets to the highest boxes during training, and its curriculum is still rising, while those of \gls{crl} and \gls{accrl} peak early and then get worse or stagnate.
These results are a first step toward self-supervised \gls{rl} for real robotics tasks, but they also show where the current limitations are.
Without an additional cost for walking forward, \ours{} learns to reach goal positions by walking backward, and since we found that this cost keeps the box-climbing policies below a few centimeters, we remove it for that task and the policies climb the box backward.
The tracking error is also not yet at the level of a reward-based PPO \citep{schulman2017proximal}, which reaches a tracking error of only $0.07$ meters per second, and in further tests with the Unitree G1 humanoid, the policies did not learn to walk stably at all.
Also the trot-shaped gait prior, while not harming sim-to-real transfer, was necessary to get the learning going in the first place.
These are all issues that current reward-based methods overcome easily in practice, which opens up clear directions for future research on contrastive and survival-based self-supervised \gls{rl} methods.

\begin{figure}[t]
\centering
\includegraphics[width=\linewidth]{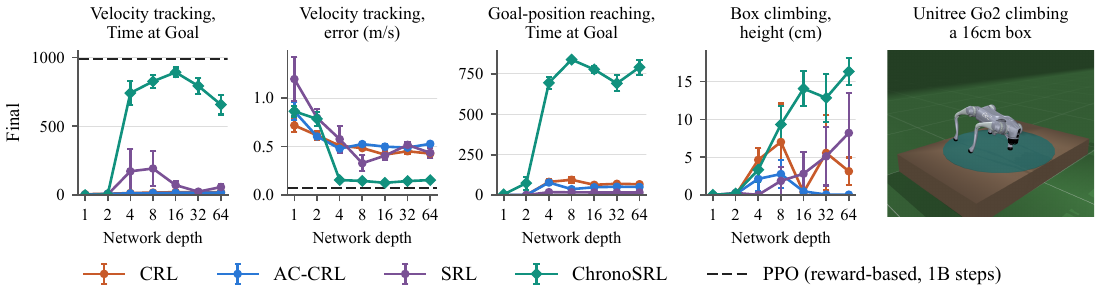}
\caption{\ours{} learns velocity tracking and goal-position reaching with the Go2 quadruped and, from eight layers on, ends training with the highest box.
The panels show the final \ToG{} out of $1000$ steps, tracking error, and curriculum box height as mean over four seeds with one standard error.
We add a PPO baseline trained for 1B steps on the velocity tracking task for comparison.
The image shows a $64$-layer \ours{} policy climbing a $16$cm box, with the goal-region in cyan.}
\label{fig:robotics}
\end{figure}

\section{Discussion and limitations}
\label{sec:limits}
Since the goal-reaching times come from the agent's own behavior, the geometry of \ours{} describes what the agent can currently do rather than the shortest paths that are possible.
This makes it a natural basis for planning with subgoals that the agent can actually reach, which we leave for future work.
Among the components of \ours{}, the geometric losses speed up learning most on the hardest tasks, and the survival critic that we build around them, with our occupancy head and action chunks, accounts for most of the gain over \gls{srl}.
Predicting how long the agent stays at the goal, and not only when it arrives, proved important on Humanoid and the quadruped, and we expect it to matter for many practical tasks that reward staying rather than arriving, such as holding a pose or a velocity.

Our temporal supervision relies on trajectories that make progress toward their relabeled goals.
Such progress is rare in the arm manipulation tasks of JaxGCRL, where the goal is the position of an object that rarely moves, so both survival critics, \ours{} and \gls{srl}, fail there, unlike the contrastive critics (see Appendix~\ref{app:boundary} for the arm results).
Here, goals that also contain the position of the gripper could make progress visible, since the gripper moves toward the object long before the object itself moves.
The same temporal geometry could also guide exploration by proposing goals at the edge of what the agent can currently reach \citep{pitis2020maximum}.
Beyond the observation window, the critic only knows that a goal is far away, so tasks with longer horizons may benefit or even strictly require bootstrapping goal-reaching times beyond $W$ steps.

\Needspace{11\baselineskip}
\section{Conclusion}
We introduced \ours{}, a self-supervised \gls{rl} method whose critic represents goal-reaching time as distance.
It trains these distances to match the times that every trajectory provides for free, and it predicts from the same embeddings when the goal will be reached and how long the agent will stay there.
On seven locomotion and navigation benchmarks, \ours{} typically reaches the final performance of contrastive, action-chunked contrastive, and survival baselines in less than a third of their training budget, and even with much smaller networks, it learns faster than baselines with up to $64$ layers.
On realistic and challenging quadruped locomotion tasks, it is the only one of these methods that learns to stay at commanded velocities and goal positions, and climbs the highest boxes.
\ours{} shows that the time at which a goal is reached is not only what an agent should minimize, but also a signal that should directly shape the representation it learns.

\label{maintext:end}

\subsection*{Reproducibility statement}
Section~\ref{sec:method} and Appendix~\ref{app:objective_details} describe the full method, Algorithm~\ref{alg:chronosrl} summarizes the training procedure, and Table~\ref{tab:configuration} lists all hyperparameters.
Appendix~\ref{app:implementation} specifies the network architectures, the training setup, and the configurations of the baselines, and Appendix~\ref{app:provenance} defines the evaluation protocol, including seeds and aggregation.
Appendix~\ref{app:robots} specifies the robot tasks and their costs, and Appendix~\ref{app:theory} contains the derivations behind our temporal targets.
\ifanonymous
All results use four seeds per setting, and we will release the code for all methods, environments, and experiments upon acceptance.
\else
All results use four seeds per setting, and the code for all methods and tasks is available through the link on our project website, \website.
\fi

\subsection*{Use of large language models}
We used large language models, including OpenAI Codex and Anthropic's Claude, as research and writing assistants.
They helped develop and critique ideas, start experiments, process and analyze experiment logs and draft the figures.
The authors reviewed all content and take full responsibility for the paper, including all material produced with the help of language models.

\ifanonymous\else
\subsection*{Acknowledgments}
This project was funded by National Science Centre Poland in the Weave programme UMO-2021/43/I/ST6/02711, and by the German Science Foundation (DFG) under grant number PE 2315/17-1.
We acknowledge EuroHPC Joint Undertaking for awarding us access to LUMI at CSC, Finland, MeluXina at LuxProvide, Luxembourg, Vega at IZUM, Slovenia, and Arrhenius at NAISS, Sweden, through the Benchmark Access project EHPC-BEN-2026B08-075 for LUMI and the Development Access project EHPC-DEV-2026D08-135 for the other three systems.
LUMI is hosted by CSC (Finland) and the LUMI consortium.
The authors gratefully acknowledge the HPC RIVR consortium and EuroHPC JU for funding this research by providing computing resources of the HPC system Vega at the Institute of Information Science.
Some of the experiments were performed on the Luxembourg national supercomputer MeluXina, and the authors gratefully acknowledge the LuxProvide teams for their expert support.
Computational resources were provided by the National Academic Infrastructure for Supercomputing in Sweden (NAISS), funded by the Swedish Research Council.
The authors gratefully acknowledge the computing time granted by the KISSKI project.
The calculations for this research were conducted with computing resources under the project kisski-merl.
The authors gratefully acknowledge the computing time provided to them on the high-performance computer Lichtenberg II at the NHR Center NHR4CES at TU Darmstadt (project number p0028682).
This is funded by the Federal Ministry of Research, Technology and Space, and the state governments participating on the basis of the resolutions of the GWK for national high performance computing at universities.
Some of the experiments ran on the hessian.AI cluster 43, and we gratefully acknowledge support from the hessian.AI Service Center (funded by the Federal Ministry of Research, Technology and Space, BMFTR, grant no.~16IS22091) and the hessian.AI Innovation Lab (funded by the Hessian Ministry for Digital Strategy and Innovation, grant no.~S-DIW04/0013/003).
\fi

\bibliography{references}
\bibliographystyle{iclr2027_conference}

\clearpage
\appendix

\section{Temporal distance and action value}
\label{app:theory}
\label{app:waiting_value}
Here, we explain why \ours{} does not use its temporal distances directly as action values.
Throughout, we fix a state, an action chunk, a goal, and the policy that acts after the chunk, and we omit these fixed arguments by writing $T\in\{0,1,\ldots\}\cup\{\infty\}$ for the first goal-reaching time.

We start from the hazard value of Equation~\ref{eq:arrival}, which counts the discounted steps spent waiting for the goal within the window.
For every $T$, this number is
\begin{equation}
\sum_{t=0}^{W-1}\gamma^t\,\mathbf 1\{T>t\}=\frac{1-\gamma^{\min(T,W)}}{1-\gamma}.
\label{eq:waiting_identity}
\end{equation}
Taking the expectation under the true survival function $S(t)=\Pr(T>t)$ then turns the hazard value into $Q_{\mathrm{hazard}}=-(1-V)/(1-\gamma)$ with $V=\mathbb E[\gamma^{\min(T,W)}]$.
This is the finite-window form of the value identity of \citet[Proposition~4.1]{nguimatsiatiofack2026svl} used by \gls{srl}, and because the sum is finite, the identity also holds when the goal may never be reached.
Since the value thus increases with $V$, the following proposition takes a closer look at $V$ and splits it into the probability of reaching the goal within the window and the time it takes when the goal is reached.

\begin{proposition}[Probability and time of goal reaching]
\label{prop:bound}
Let $0<\gamma<1$, $\kappa=-\log\gamma$, $p=\Pr(T\le W)>0$, and $\mu=\mathbb E[T\mid T\le W]$. Then
\begin{equation}
V=p\,\mathbb E\big[e^{-\kappa T}\mid T\le W\big]+(1-p)\,e^{-\kappa W}
\;\ge\;p\,e^{-\kappa\mu}+(1-p)\,e^{-\kappa W},
\label{eq:bound}
\end{equation}
with equality if and only if $T$ is almost surely constant given $T\le W$.
\end{proposition}
\begin{proof}
Since $\min(T,W)=T$ on the event $T\le W$ and $\min(T,W)=W$ on the event $T>W$, splitting the expectation over these two events gives the equality.
For the inequality, note that the function $t\mapsto e^{-\kappa t}$ is strictly convex because $\kappa>0$.
Jensen's inequality therefore gives $\mathbb E[e^{-\kappa T}\mid T\le W]\ge e^{-\kappa\mu}$, with equality if and only if $T$ is almost surely constant given $T\le W$, and multiplying by $p>0$ yields the bound together with its equality condition.
\end{proof}

The time loss of Equation~\ref{eq:time}, in contrast, captures only the second of these two quantities, because it uses only the examples in which the goal was reached.
For a fixed input and an unrestricted prediction, its minimizer therefore depends on the goal-reaching times of these examples but not on the probability $p$ of reaching the goal.
If, in addition, all residuals lie within one discount horizon, where the Huber loss is quadratic, this minimizer is simply the mean of $\kappa T$ over the reached examples.

To make this difference between the distance and the value concrete, we consider two actions that reach the goal after exactly $200$ steps whenever they reach it, where the first action reaches the goal in $90\%$ of the attempts and the second in only $10\%$.
Both actions thus have the same target distance $\kappa\cdot200\approx0.2$, yet their values differ.
At $\gamma=0.999$ and $W=1000$, Equation~\ref{eq:bound} holds with equality and gives $V\approx0.77$ for the first action and $V\approx0.41$ for the second, which correspond to $Q_{\mathrm{hazard}}\approx-226$ and $Q_{\mathrm{hazard}}\approx-587$ discounted steps of waiting.
A value computed from the distance alone, such as $e^{-d}$, nevertheless ranks both actions equally.
Even for certain arrivals, such a value recovers $V$ only when the goal-reaching time does not vary, since the inequality in Equation~\ref{eq:bound} is strict otherwise.
The hazard head relies on neither assumption.
Instead, it models the distribution of $T$ within the window, and its likelihood also learns from the censored examples, which carry the information about $p$.

\section{Labels, losses, and prediction heads}
\label{app:objective_details}
Algorithm~\ref{alg:chronosrl} summarizes the training of \ours{}, and the subsections below describe the labels, losses, and prediction heads it relies on.

\begin{algorithm}[ht]
\caption{\ours{}}
\label{alg:chronosrl}
\begin{algorithmic}[1]
\State Initialize the critic, the policy $\pi$, the entropy coefficient $\alpha$, and the replay buffer
\For{each training step}
\State Collect $62$ steps in every parallel environment with $\pi$, executing chunks of two actions, and store the transitions
\State Sample one window of $W$ consecutive transitions per environment from the replay buffer
\State Relabel goals and compute the labels of every example (see Appendix~\ref{app:labels})
\For{each minibatch of relabeled examples}
\State Update the critic on $\mathcal L_{\mathrm{hazard}}+\mathcal L_{\mathrm{occ}}+\mathcal L_{\mathrm{time}}+\mathcal L_{\mathrm{sep}}$ (Section~\ref{sec:readout})
\State Compute $\beta$ on the minibatch and update $\pi$ on $J(\pi)$ (Equation~\ref{eq:actor}), then update $\alpha$
\EndFor
\EndFor
\end{algorithmic}
\end{algorithm}

\subsection{Relabeling and goal-reaching times}
\label{app:labels}
A window holds $W=1000$ consecutive transitions, so it can contain the end of one episode and the beginning of the next.
Within such a window, the relabeler follows \gls{srl} and gives each step $i$ a future goal with probability $0.85$, the current achieved goal with probability $0.05$, and the achieved goal of a random step of the window with probability $0.10$.
Future goals come from the later steps $j$ of the same episode with probability proportional to $\gamma_g^{\,j-i}$, where $\gamma_g=0.99$ is the goal discount of \gls{crl}, and the current goal is used whenever the episode has no later step.
Once a goal is chosen, the goal-reaching time $\tau_i$ is the first offset after step $i$ at which the episode enters $\mathcal B_\varepsilon(g_i)$, with $\tau_i=0$ if $s_i$ already lies in the goal region.
The indicator $\delta_i$ is one if this offset lies within the window, and otherwise the example is censored at the full window $W$, as in \gls{srl}.

Following \gls{srl}, the humanoid tasks replace the point goal by a sequence of $k=32$ consecutive achieved goals from the sampled step onward, and a random goal is repeated $k$ times.
A sequence is matched at a start step if the $k$ achieved goals from there lie within $\varepsilon$ of the corresponding elements, all within the same episode, and the time $\tau_i$ is then the earliest matching start relative to step $i$.
The policy, however, observes only a point goal $g$ and queries the critic with the repeated sequence $(g,\ldots,g)$.
Finally, an example is valid only if its chunk of two actions from step $i$ onward lies within the window and the episode.

\subsection{Time and separation losses on a batch}
\label{app:geometry_losses}
Both geometric losses are computed on a batch, where $D_{ij}=\dist(s_i,a_i,g_j)$ pairs the state and action of example $i$ with the goal of example $j$ and $v_i$ indicates a valid example.
The time loss then averages over the reached examples $\mathcal R=\{i\mid v_i=1,\ \delta_i=1,\ \tau_i\ge1\}$,
\begin{equation}
\mathcal L_{\mathrm{time}}=\frac{1}{\max(1,|\mathcal R|)}\sum_{i\in\mathcal R}\rho\big(D_{ii}-\kappa\tau_i\big),
\qquad
\rho(r)=\begin{cases}\tfrac12r^2,&|r|\le1,\\|r|-\tfrac12,&|r|>1.\end{cases}
\label{eq:time_batch}
\end{equation}
The separation loss, in contrast, penalizes distances below the margin $\kappa W$ and places the censored examples $\mathcal C=\{i\mid v_i=1,\ \delta_i=0\}$ and the cross-batch pairs $\mathcal M=\{(i,j)\mid i\neq j,\ \Delta(g_i,g_j)>\varepsilon\}$ under one denominator,
\begin{equation}
\mathcal L_{\mathrm{sep}}=\frac{\sum_{i\in\mathcal C}[\kappa W-D_{ii}]_++\sum_{(i,j)\in\mathcal M}[\kappa W-D_{ij}]_+}{\max\big(1,|\mathcal C|+|\mathcal M|\big)}.
\label{eq:separation_batch}
\end{equation}
Here the goal difference $\Delta$ is the Euclidean distance between point goals and, for goal sequences, the root mean square of the element-wise distances.
Because the cross-batch pairs are selected by their goal difference alone, they may contain reachable goals, which is why they never enter the hazard likelihood.
In both losses, distances are computed as $\big(\|\phi_\theta(s,a)-\psi_\theta(g)\|_2^2+10^{-12}\big)^{1/2}$.

\subsection{Hazard head}
\label{app:survival}
We use the hazard head of \gls{srl} without changes \citep{nguimatsia2026survival}.
It predicts the probability that the goal is already reached together with one hazard for every step of the window, and with one bin per step the piecewise-constant survival parameterization of \citet{nguimatsiatiofack2026svl} is exact.
The likelihood then treats the two kinds of examples differently.
A reached example contributes the log-probability of reaching the goal at step $\tau$, which for $\tau=0$ is the probability that the goal is already reached, whereas a censored example contributes the log-probability of not reaching the goal within the window.
The loss $\mathcal L_{\mathrm{hazard}}$ averages the negative log-likelihood over the valid examples, and the value of Equation~\ref{eq:arrival} likewise adds no cost for the time after $W$.

\subsection{Occupancy targets and value}
\label{app:occupancy}
The occupancy head is one additional linear output layer on the prediction network of the hazard head (see Appendix~\ref{app:architecture}).
Its targets record how close the trajectory stays to the goal over time.
For the goal $g_i$ of example $i$, or the first element of a goal sequence, the target at an offset $t\ge1$ averages the membership in three nested goal regions,
\begin{equation}
o_{it}=\frac13\sum_{r\in\{\varepsilon,\,\varepsilon/2,\,\varepsilon/4\}}\mathbf 1\big\{s_{i+t}\in\mathcal B_r(g_i)\big\}.
\label{eq:nested_occupancy}
\end{equation}
Whether a target is observed depends on how the episode of step $i$ ends.
The target is observed while this episode continues within the window, and after a termination all later offsets of the window are observed with $o_{it}=0$.
Offsets after a truncation or beyond the window, by contrast, are unobserved.
Rather than predicting every offset, the head groups the offsets into $L=30$ bins with logarithmically spaced integer boundaries $0=b_0<b_1<\cdots<b_L=W$, and the bin target $y_{i\ell}$ is the mean of the observed $o_{it}$ in bin $\ell$.
Bins without an observation are masked.
With the predicted logit $z_{i\ell}$ and the mask $m_{i\ell}$ of valid examples and observed bins, the occupancy loss is a binary cross-entropy that weighs every observed bin equally, independently of its width,
\begin{equation}
\mathcal L_{\mathrm{occ}}=\frac{\sum_{i,\ell}m_{i\ell}\big[\operatorname{softplus}(z_{i\ell})-y_{i\ell}\,z_{i\ell}\big]}{\max\big(1,\sum_{i,\ell}m_{i\ell}\big)}.
\label{eq:occupancy_batch}
\end{equation}
To obtain the occupancy value of Equation~\ref{eq:occupancy_value}, we hold the prediction constant within each bin and leave out the first bin, which only contains the offset zero,
\begin{equation}
Q_{\mathrm{occ}}=\sum_{\ell=1}^{L-1}\sigma(z_\ell)\sum_{t=b_\ell}^{b_{\ell+1}-1}\gamma^t,
\label{eq:occupancy_bins_value}
\end{equation}
where $\sigma$ is the logistic function.

\subsection{Occupancy weight and policy update}
\label{app:actor_details}
In the policy objective, the occupancy value is scaled by the weight $\beta$, the fraction of valid examples in the batch whose episode terminates within the window at or after their step,
\begin{equation}
\beta=\frac{\sum_iv_ie_i}{\max\big(1,\sum_iv_i\big)},
\label{eq:occupancy_weight}
\end{equation}
where $e_i$ indicates such a termination.
This weight is computed without gradient and held fixed during the policy update.
It only scales $Q_{\mathrm{occ}}$ in $J(\pi)$, whereas the occupancy head itself is trained in every update.
The policy is a tanh-squashed diagonal Gaussian over all components of a chunk, and its entropy coefficient is tuned toward a target entropy of $-0.5$ per action dimension of the chunk \citep{haarnoja2018soft}.
When the policy is updated, the gradient with respect to the sampled chunk passes through both encoders and both heads of the fixed critic.

\subsection{Shaping costs}
\label{app:shaping}
Shaping replaces the plain step count by a clock that also accumulates costs.
For a trajectory with per-step costs $c_t\ge0$, the shaped clock starts at $C_0=0$ and advances as $C_{t+1}=C_t+1+\lambda c_t$, and the shaped time $\widetilde\tau$ from step $i$ to a goal reached at step $j$ is $C_j-C_i$, rounded to the nearest integer.
A goal then counts as reached only if its rounded shaped time is at most $W$.
With $\lambda>0$, an example without such an arrival is censored at the rounded shaped time until the end of its window, clipped to $W$, while the separation margin remains $\kappa W$.
Occupancy observations are likewise assigned to bins by their rounded shaped time, and those beyond the last boundary are clipped into the last bin.
They are neither interpolated nor weighted by the duration of their cost increment.
The costs can take two forms.
They can be the mean squared change between consecutive stored actions, which is zero for the first action of an episode, or the nonnegative weighted sum of the shaping penalties that the robot environment reports.
Shaped time is implemented in the relabeling of goal sequences, which the robot tasks use, and with $\lambda=0$ all labels, including the censoring at $W$, are exactly the unshaped labels (see Appendix~\ref{app:labels}).

\section{Architecture and training}
\label{app:implementation}

\subsection{Networks}
\label{app:architecture}
All networks follow the residual architecture of scaled \gls{crl} and \gls{srl} \citep{wang2025scaling,nguimatsia2026survival}, in which each residual block contains four dense layers with layer normalization and Swish activations.
As in \citet{wang2025scaling}, the reported depth counts the dense layers in the residual blocks of one network.
Because a residual block needs four layers, networks with one or two layers are plain stacks of such layers instead.

The critic first embeds its inputs with two encoders.
The state--action encoder $\phi_\theta$ receives the state and the flattened action chunk, whereas the goal encoder $\psi_\theta$ receives the goal or the flattened goal sequence.
Both encoders have the reported depth and end in a linear projection to $64$ dimensions without layer normalization.
Their embeddings then enter the prediction network, which is the hazard network of \gls{srl} \citep{nguimatsia2026survival}.
This network modulates the state--action embedding with the goal embedding, processes both with a residual network of the reported depth, and outputs the instant probability and the hazards through a learned temporal basis of rank $64$.
For the occupancy head, we add a linear layer with $30$ outputs to this network, so that both heads share all parameters except their output layers.
The actor, finally, is a separate residual network of the reported depth that maps the state and the point goal to a Gaussian over action chunks.

\subsection{Training configuration}
\label{app:training}
Table~\ref{tab:configuration} lists all training settings of \ours{} next to those of the reference implementation of \gls{srl}.
We set all loss weights to one and tune no setting per task.

\begin{table}[ht]
\centering\small
\caption{On the seven benchmarks, \ours{} differs from the reference implementation of \gls{srl} only in the settings of the upper block.}
\label{tab:configuration}
\setlength{\tabcolsep}{5pt}
\begin{tabular}{@{}lll@{}}
\toprule
Setting & \gls{srl} & \ours{} \\
\midrule
Critic loss & $\mathcal L_{\mathrm{hazard}}$ & $\mathcal L_{\mathrm{hazard}}+\mathcal L_{\mathrm{occ}}+\mathcal L_{\mathrm{time}}+\mathcal L_{\mathrm{sep}}$ \\
Value used by the policy & $Q_{\mathrm{hazard}}$ & $Q_{\mathrm{hazard}}+\beta\,Q_{\mathrm{occ}}$ \\
Action chunk & One action & Two actions, open-loop \\
Future-goal sampling discount & $0.999$ & $0.99$ \\
Embedding dimension & $128$, layer normalization & $64$, no normalization \\
Encoder depth & Half the reported depth & The reported depth \\
Occupancy bins and radii & None & $30$ bins, radii $\varepsilon$, $\varepsilon/2$, $\varepsilon/4$ \\
\midrule
Budget and episode length & \multicolumn{2}{l}{$100$M environment steps, episodes of $1000$ steps} \\
Seeds and depths & \multicolumn{2}{l}{$1000$ to $1003$, depths $1$, $2$, $4$, $8$, $16$, $32$, and $64$} \\
Parallel environments and unroll length & \multicolumn{2}{l}{$512$ environments, $62$ steps} \\
Minibatch updates per training step & \multicolumn{2}{l}{Up to $800$} \\
Replay size per environment & \multicolumn{2}{l}{$1000$ to $10000$ transitions} \\
Learning rates & \multicolumn{2}{l}{$3\times10^{-4}$ for actor, critic, and entropy coefficient} \\
Hidden width & \multicolumn{2}{l}{$256$} \\
Discount and observation window & \multicolumn{2}{l}{$\gamma=0.999$ and $W=1000$ steps} \\
Hazard bins & \multicolumn{2}{l}{$1000$ bins of one step} \\
Goal radius & \multicolumn{2}{l}{$\varepsilon=0.5$} \\
Goal mixture & \multicolumn{2}{l}{$0.85$ future, $0.05$ current, $0.10$ random} \\
Humanoid goals and batch size & \multicolumn{2}{l}{Sequences of $k=32$ goals, $1024$ examples} \\
Ant goals and batch size & \multicolumn{2}{l}{Single goals, $512$ examples} \\
\bottomrule
\end{tabular}
\end{table}

\subsection{Reference methods}
\label{app:reference_methods}
All baselines run in our training pipeline with the reference configuration of their papers.
Within this pipeline, all methods perform up to $800$ updates for every $512\times62$ environment steps on minibatches of $512$ examples, or $1024$ for the survival critics on the humanoid tasks.
Both \gls{crl} and \gls{accrl} use the forward InfoNCE objective with a log-partition penalty of $0.1$, the discount $0.99$, and embeddings with $64$ dimensions, and \gls{accrl} additionally executes chunks of three actions.
For \gls{srl}, we use the configuration of Table~\ref{tab:configuration}.
Since every baseline keeps the configuration of its paper, the comparisons with the complete baselines also include differences in architecture and goal sampling.
The encoders of \gls{srl}, for example, have only half the reported depth.
The advantage of a four-layer \ours{} critic over baselines with $64$ layers, however, exceeds this factor by far.

We simply chose a budget of $100$ million steps to keep the compute time of all experiments manageable.
It is smaller than the $200$ to $400$ million steps that \citet{wang2025scaling} use for Ant Hardest Maze and the humanoid mazes.
This smaller budget explains the low absolute \ToG{} of all methods on the humanoid mazes, and with longer training all methods, including \ours{}, would likely improve there.
Since \ours{} needs fewer layers, it also reaches a given performance at a lower cost.
A four-layer run of \ours{} in the ant mazes, for example, takes $1.8$ hours on GH200 GPUs, against $8.1$ to $9.4$ hours for the baselines with $64$ layers on the same GPUs.

\section{Evaluation protocol}
\label{app:provenance}
We train all four methods on the seven tasks at depths from $1$ to $64$ with the seeds $1000$ to $1003$, and each run lasts $100$ million environment steps.
Every million steps, we evaluate the policy on $100$ episodes of $1000$ steps.

We summarize each run by its normalized AUC and its final value.
The normalized AUC is the mean of evaluations $1$ to $100$ and equals the area under the piecewise-constant learning curve divided by the training budget, whereas the final value is the mean of evaluations $96$ to $100$.
The learning speed in the last panel of Figure~\ref{fig:depth_auc} instead compares \ours{} with the baselines.
It is the number of environment steps after which the mean \ToG{} of \ours{} over four seeds, smoothed over three evaluations, first reaches the mean final \ToG{} of the strongest baseline at the same depth.
Note that the highest possible \ToG{} lies below $1000$, since the agent first has to travel to the goal.
In Ant U4-Maze, for example, where the final policies of \ours{} need a median of about $130$ steps to reach the goal, it lies at about $870$.

We compute each summary per seed and report the mean over the four seeds together with the standard deviation or the standard error over seeds.
Tables give means and sample standard deviations, and plots give means and standard errors.

\section{Benchmark results}
\label{app:full_results}
Figures~\ref{fig:depth_auc} and~\ref{fig:depth_final} plot the normalized AUC and the final \ToG{} against network depth (see Appendix~\ref{app:provenance}).
Once the networks have eight or more layers, \ours{} has the highest normalized AUC on every task.
Its mean final \ToG{} is also the highest in every setting except Ant U5-Maze with $64$ layers, where \gls{srl} is slightly ahead with $768$ against $756$ steps.
Beyond these two metrics, Table~\ref{tab:reach} shows that, with eight layers, \ours{} reaches the goal in almost every episode of the four ant mazes.

\begin{figure}[!htbp]
\centering
\includegraphics[width=\linewidth]{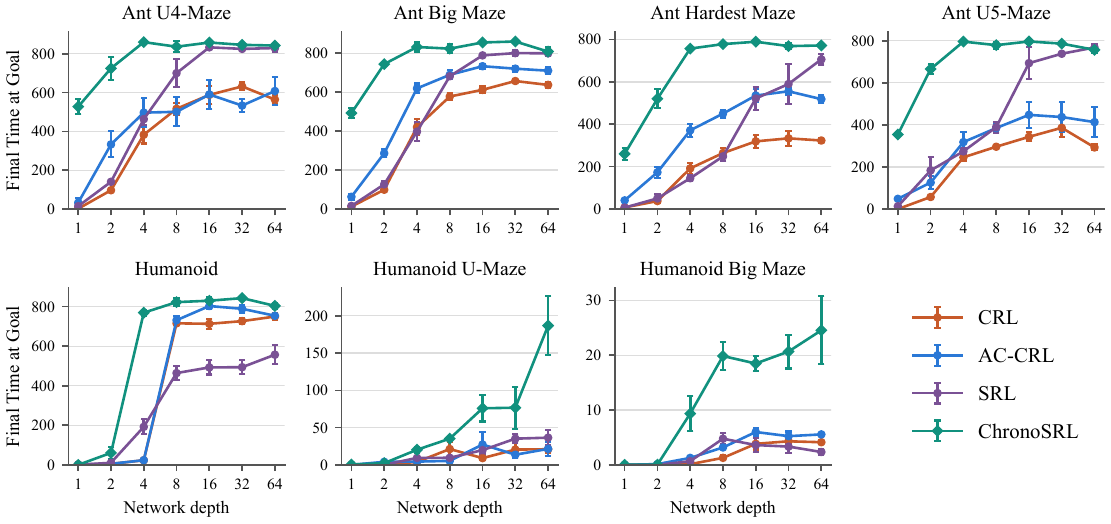}
\caption{Points show the mean final \ToG{} over four seeds with one standard error.}
\label{fig:depth_final}
\end{figure}
\begin{table}[!htbp]
\centering\small
\caption{Each cell gives the percentage of evaluation episodes in which an eight-layer agent reaches the goal at least once, averaged over the last five evaluations and four seeds.}
\label{tab:reach}
\begin{tabular}{lcccc}
\toprule
Task & CRL & AC-CRL & SRL & \ours{} \\
\midrule
Ant U4-Maze & 67 & 66 & 85 & 99 \\
Ant Big Maze & 81 & 92 & 88 & 98 \\
Ant Hardest Maze & 41 & 70 & 34 & 99 \\
Ant U5-Maze & 39 & 56 & 50 & 99 \\
Humanoid & 93 & 94 & 94 & 94 \\
\bottomrule
\end{tabular}
\end{table}
\clearpage

\subsection{Learning curves at every depth}
\label{app:learning_curves}
Figure~\ref{fig:learning_grid} complements these summaries with the learning curves of all methods at every depth.
\begin{figure}[!htbp]
\centering
\includegraphics[width=\linewidth]{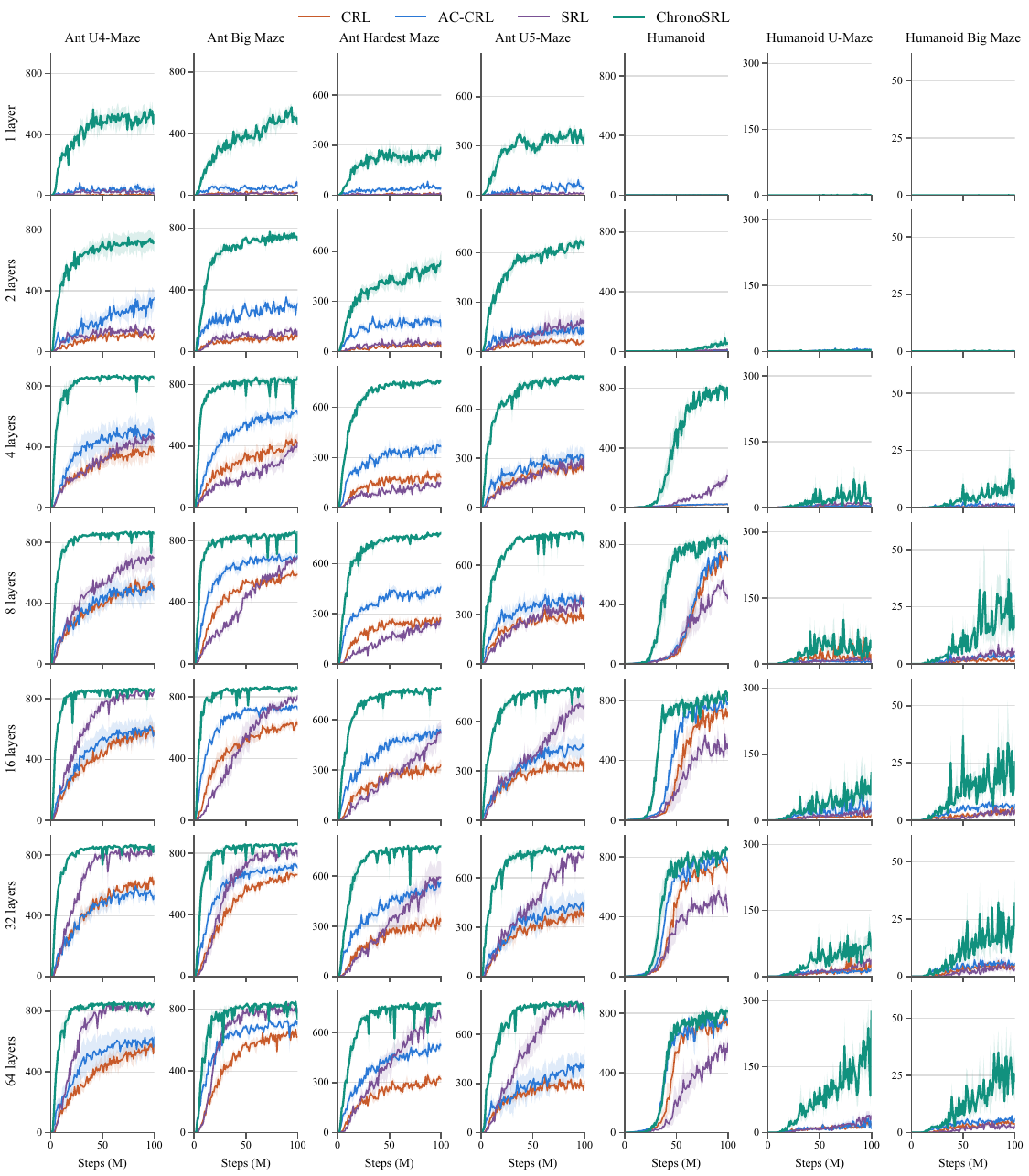}
\caption{Rows show the network depth and columns the task, and each line shows the mean \ToG{} over four seeds with one standard error.}
\label{fig:learning_grid}
\end{figure}

\clearpage

\section{Component ablations}
\label{app:ablation_protocol}
To see what individual components contribute, we ablate \ours{} on Humanoid with $16$ layers and on Ant Hardest Maze and Ant U4-Maze with $8$ layers.
Each ablation removes one component and keeps all other settings and the seeds of \ours{}.
The matched base is \ours{} without the time and separation losses, again with all other settings unchanged.
Rather than removing the time loss, the variant with permuted goal-reaching times shuffles its targets among the reached examples of each batch, which keeps their distribution but removes their relation to the examples.
The variant without the occupancy head drops it from both the critic loss and the value of the policy.
The distance actor, in contrast, keeps all critic losses and maximizes $-\dist(s,a,g)/(1-\gamma)$ with the same goal query.
Finally, the one-step variant executes single actions but still counts its labels in environment steps.
As for their effects, removing the occupancy head lowers the final \ToG{} on Humanoid from $830$ to $509$, because the humanoid falls after reaching its goal.
Single actions, on the other hand, slow down learning, so that \ours{} reaches $80\%$ of its final \ToG{} only after $56$ instead of $18$ million steps on Ant Hardest Maze and after $29$ instead of $9$ million steps on Ant U4-Maze.

\subsection{Geometric losses in another critic}
To test whether the geometric losses also help another critic, we add them with weight one to \gls{srl} in its complete reference configuration and train it on the seeds of our \gls{srl} baseline.
In particular, this configuration keeps the normalized embeddings, half-depth encoders, goal sampling, single actions, and value of \gls{srl}.
Figure~\ref{fig:srlgeo} shows that the losses speed up \gls{srl} in all four ant mazes with $8$ layers, and in each of these mazes every seed with the losses has a higher AUC than every seed without them.
On Ant Hardest Maze, the losses more than double the normalized AUC of \gls{srl}, from $143$ to $335$, and in the other three mazes \gls{srl} with the losses even reaches the final \ToG{} of \ours{}.
On Humanoid with $16$ layers, however, three of the four seeds of \gls{srl} with the losses do not learn, which we attribute to a conflict between the losses and the normalized embeddings of \gls{srl}.
Still, the survival critic of \ours{} alone, without the geometric losses, already improves on \gls{srl} more than the losses do on each of these tasks.
On Ant Hardest Maze, for example, this matched base reaches a normalized AUC of $647$.
\begin{figure}[ht]
\centering
\includegraphics[width=\linewidth]{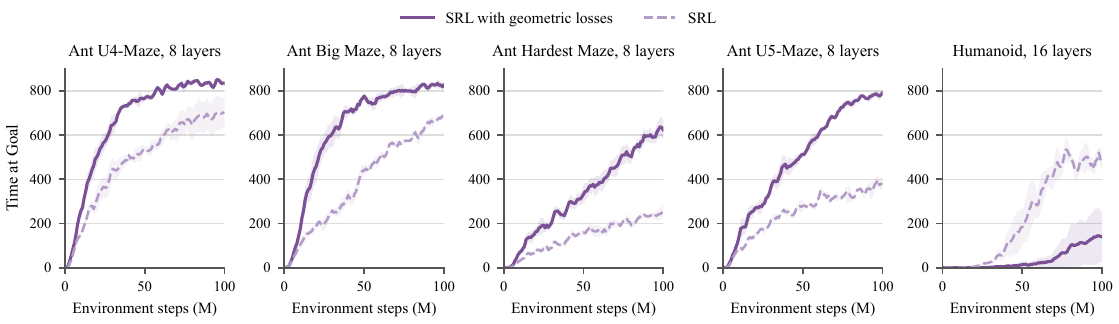}
\caption{Added to \gls{srl} in its reference configuration, the geometric losses speed up learning in all four ant mazes, but three of the four seeds do not learn on Humanoid.
Lines show the smoothed mean \ToG{} over four seeds with one standard error.}
\label{fig:srlgeo}
\end{figure}

\subsection{Results of the ablations}
Table~\ref{tab:ablation_auc} gives the normalized AUC of all ablations alongside that of \ours{}.
Without either geometric loss, the AUC drops on Humanoid and Ant Hardest Maze, whereas on Ant U4-Maze both variants learn equally fast.
Without the occupancy head, the humanoid still reaches the goal in $98\%$ of the evaluation episodes, yet it loses $426$ of the $1000$ steps after a fall, against $106$ for \ours{}.
A policy that follows the distance, by contrast, hardly learns on any task.
Finally, single actions lower the AUC on every task, most of all on Ant Hardest Maze.
\begin{table}[ht]
\centering\small
\caption{Following the distance fails on every task, while single actions slow down learning everywhere.
On Humanoid, every ablation lowers the AUC, and removing the occupancy head nearly halves it.
Entries show the normalized AUC as mean $\pm$ sample standard deviation over four seeds.}
\label{tab:ablation_auc}
\begin{tabular}{lccc}
\toprule
Variant & Humanoid, 16 layers & Ant Hardest Maze, 8 layers & Ant U4-Maze, 8 layers \\
\midrule
\ours{} & 541 $\pm$ 14 & 682 $\pm$ 16 & 801 $\pm$ 18 \\
Neither geometric loss & 508 $\pm$ 39 & 647 $\pm$ 42 & 800 $\pm$ 11 \\
Permuted times & 505 $\pm$ 45 & 674 $\pm$ 11 & 800 $\pm$ 21 \\
No occupancy head & 286 $\pm$ 22 & 681 $\pm$ 19 & 801 $\pm$ 6 \\
One-step actions & 475 $\pm$ 51 & 486 $\pm$ 33 & 710 $\pm$ 14 \\
Distance as value & 17 $\pm$ 7 & 70 $\pm$ 17 & 70 $\pm$ 15 \\
\bottomrule
\end{tabular}
\end{table}

\section{Measuring temporal geometry}
\label{app:measurements}
Here, we describe how we measure the temporal geometry analyzed in Section~\ref{sec:geometry}.
We run the final deterministic policies of all four seeds and record the distance to the goal at every decision step before the first arrival.
After converting this distance to steps by dividing it by $\kappa$, we compare it with the steps that remain until the goal is first reached.
Throughout, all correlations are Spearman rank correlations, and values with $\pm$ are means and standard errors over four seeds.

\subsection{Distance maps}
To build the distance maps, we place the ant at the center of every free cell of the maze, with the reset noise of the environment and a uniform offset of up to half a meter.
In Ant Hardest Maze, the goal is the center of the far corner, and we run four episodes of $1000$ steps from every cell.
Ant U4-Maze instead has its goal at the center of the dead end, and there we run eight episodes of $500$ steps from every cell.
Each map then shows the median over all four seeds within squares of half a meter, taken along the episodes that reach the goal.

We measure how closely the distance follows time with rank correlations, for which each variant is compared with the remaining steps of its own final policy.
For the map in Figure~\ref{fig:geometry}, the rank correlation between distance and remaining steps is $0.93$ for \ours{}, whereas it is negative without the geometric losses.
To separate time from space in Ant U4-Maze, we further compute per seed the partial rank correlation of the distance with the remaining steps after removing the ranks of the straight-line distance to the goal, as well as the converse.
This shows that the distance of \ours{} follows the remaining steps, with a partial rank correlation of $0.62$, and hardly depends on the straight-line distance.
Without the geometric losses, by contrast, the distance depends on both about equally.

\subsection{Roles of the two geometric losses}
To separate the roles of the two geometric losses, we repeat this analysis for all variants on Ant Hardest Maze with four layers and on Ant U4-Maze with eight layers, and Figure~\ref{fig:geometry_components} shows the resulting maps.
The time loss is what orders the distances by time, since without it the distances in Ant Hardest Maze are unrelated to the remaining steps.
The separation loss, in turn, matters only where time and space differ.
In Ant Hardest Maze, the remaining steps and the straight-line distances are strongly related, whereas in Ant U4-Maze, where they differ, the distance without the separation loss follows the straight-line distance instead of the remaining steps.
Only the two losses together thus make the distance follow the time of the agent.
In addition, the separation loss stretches the scale, so that the distance of \ours{} overestimates the remaining steps several-fold in both mazes.

\begin{figure}[ht]
\centering
\includegraphics[width=\linewidth]{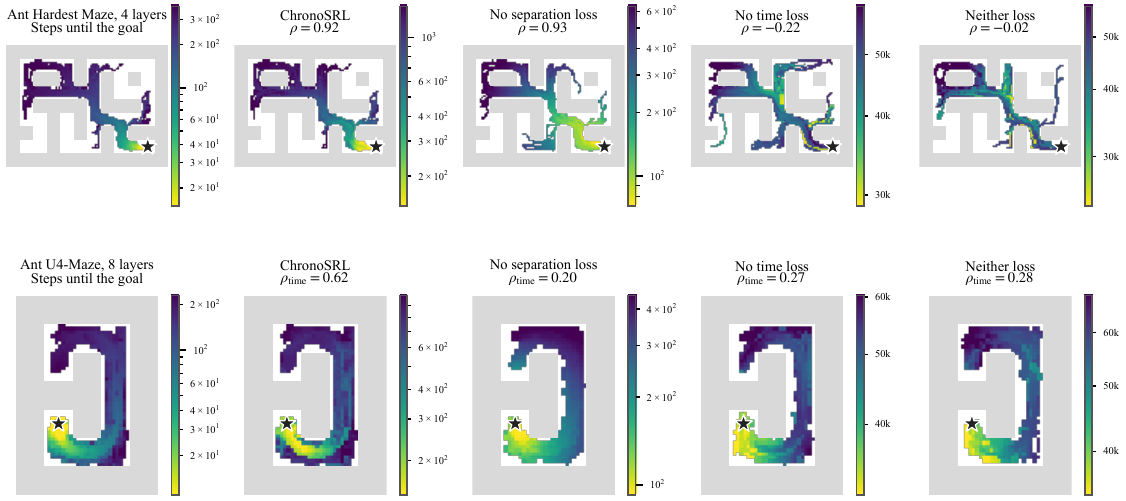}
\caption{Both geometric losses are needed for the distance to follow the time of the agent.
The maps show the steps until the goal for the final policies of \ours{} and the distance of each variant converted to steps, pooled over four seeds and each on its own color scale.
The titles give the mean rank correlation with the remaining steps, which in Ant U4-Maze is a partial correlation that accounts for the straight-line distance.}
\label{fig:geometry_components}
\end{figure}

\subsection{Hazard value along evaluation episodes}
Besides the maps, we also examine the critic along evaluation episodes.
We run $64$ evaluation episodes of \ours{} and of its variant without the geometric losses from the standard start of Ant U4-Maze and Ant Hardest Maze, and at every decision step we query the critic with the chunk that the policy plans there.
Figure~\ref{fig:clock} shows that along these episodes, the distance of \ours{} shrinks from about one discount horizon at the start to about $150$ steps, whereas without the geometric losses it stays at tens of thousands of steps.

The same queries also let us check the hazard value $Q_{\mathrm{hazard}}$ of the planned chunk, which predicts a discounted waiting time.
We compare this prediction with the waiting time that the agent actually experiences, $(1-\gamma^{T-t})/(1-\gamma)$ for a first arrival at step $T$.
In both ant mazes, the rank correlation between the two is $0.99$.
The prediction is nevertheless longer by a median of $11$ to $16\%$, since the value is trained on the stochastic behavior in the replay rather than on the deterministic evaluation policy.
\begin{figure}[ht]
\centering
\includegraphics[width=\linewidth]{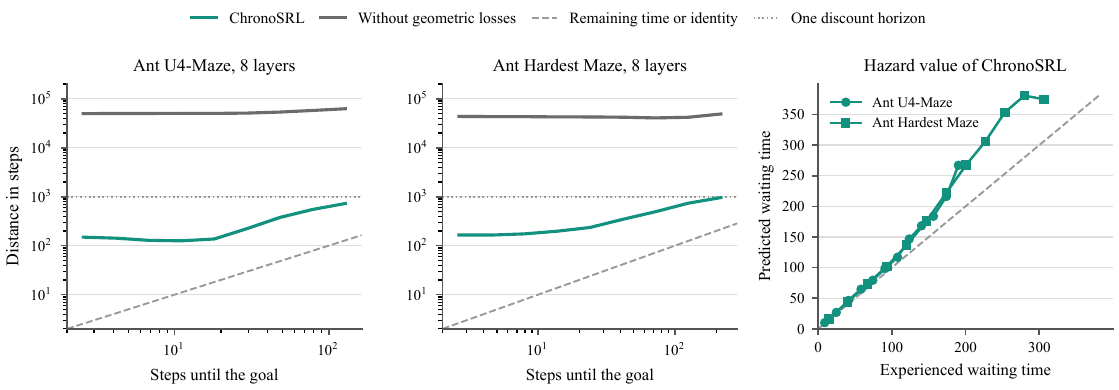}
\caption{Along evaluation episodes, the distance of \ours{} shrinks with the steps until the goal, and its hazard value tracks the discounted waiting time that the agent experiences.
The left panels show the distance in steps as the geometric mean over four seeds of the per-run medians within logarithmically spaced bins of the remaining steps, with one standard error.
The right panel shows the median predicted waiting time within bins of the experienced one, with the identity as a dashed line.}
\label{fig:clock}
\end{figure}

\section{Quadruped robot locomotion tasks}
\label{app:robots}

\subsection{Setup and tasks}
All robot experiments use the Unitree Go2 in the locomotion environment of RL-X \citep{bohlinger2024onepolicy}, which we simulate with MuJoCo-Warp in $4096$ parallel environments at $50$ Hz with episodes of $1000$ steps.
We keep the domain randomization, observation noise, shaping terms, and per-environment curriculum of this environment, and the curriculum coefficient $c$ scales both the randomization and the shaping.
The same coefficient enters the termination rule, as a training episode ends when the trunk drops below $(1-c)$ times $60\%$ of its standing height.
The actor observes the proprioception of RL-X, the phase of a gait clock, and the goal in the robot's frame from the true pose, but it receives no height map.
The actions of all methods are added to a trot-shaped prior on the joint targets with a base amplitude of $0.3$ and a frequency of $3$ Hz, and this prior stays on during evaluation.
Every run trains for $100$ epochs of $5$ million steps with batches of $8192$ examples.
Evaluation instead runs in $4096$ separate parallel environments at curriculum level zero, where an episode ends only when the robot falls.

The three tasks differ only in their goal definition.
In velocity tracking, the goal is the RL-X command for the forward, lateral, and yaw velocity of the base.
Each component of this command is sampled uniformly from $[-1,1]$ and set to zero below $0.1$, and about $8\%$ of the commands are standing commands.
A goal counts as reached when the norm of the velocity error is below $0.25$.
We report the time at the goal, the tracking error per step, and the episode length on a fixed quarter of the evaluation environments, which are commanded to walk forward at $1$ meter per second.

The other two tasks give the robot a goal pose instead.
In goal-position reaching, this pose lies $1$ to $5$ meters away in any direction, and its heading points from the start toward the goal.
We encode it as the position together with $0.3$ times the cosine and sine of the heading, and the goal counts as reached within a distance of $0.5$ in this encoding.
Box climbing adds a box centered on the goal, whose height is $36$ centimeters times the curriculum coefficient.
For goal-position reaching we report the time at the goal, whereas for box climbing we report the box height of the training curriculum, since evaluation runs at curriculum level zero.
The curriculum of each environment rises after an episode in which the trunk stood at least $20$ centimeters above the box top, within its footprint, for $50$ consecutive steps, and it falls after any other episode, with a step that grows with the number of consecutive successes or failures.

\subsection{Critics and shaping costs}
\ours{} uses its benchmark configuration and \gls{srl} its reference configuration, both with goal sequences of $k=8$.
The contrastive baselines \gls{crl} and \gls{accrl} use a discount of $0.99$, an embedding dimension of $64$, and future goals only, with chunks of three actions for \gls{accrl}.
The two survival critics of \ours{} and \gls{srl} receive the shaping terms of RL-X as the per-step cost $c_t$ of the shaped time in Section~\ref{sec:shaping}, with $\lambda=10$.
Concretely, this cost is the negated sum of the $23$ penalty terms of the RL-X reward with their weights and curriculum scaling, clipped to $[0,50]$.
In goal-position reaching and box climbing, the curriculum scale of this cost is at least $0.3$.

Goal-position reaching adds one more term, a travel cost of $0.15\,(1-\cos\vartheta)\min(v/0.2,1)$ per step that both survival critics pay and that the curriculum does not scale.
Here, $\vartheta$ is the angle between the horizontal velocity of the trunk and its heading, and $v$ is the speed in meters per second.
Without this cost, \ours{} reaches its goals by walking backward.
With it, however, the box-climbing policies flip over at the edge of the box from $4$ centimeters on, at every depth we tried.
Box climbing therefore runs without the travel cost, and its policies climb the box mostly backward.

As a reference, we also train reward-based PPO \citep{schulman2017proximal} on velocity tracking with the default hyperparameters of RL-X and an entropy coefficient of $0.002$, for about $1$ billion steps with four seeds.
PPO uses the same environment, randomization, and curriculum, but it trains on the RL-X reward without the trot prior and resamples commands within episodes.
Its final policies track the forward command with an error of $0.07$ meters per second, and after $500$ million steps, the budget of our runs, the error is already $0.085$.

\subsection{Results}
Figure~\ref{fig:robots_final} adds the episode length in velocity tracking to the final values of Figure~\ref{fig:robotics}.
It shows that from four layers on, \ours{} tracks the forward command most accurately and stays at velocity and position goals longest.
As a reference for the episode length, a robot that sends no actions to the gait prior falls after $54$ steps, but the velocity-tracking episodes of \gls{crl} and \gls{accrl} end even earlier, after $32$ to $41$ steps.
\begin{figure}[ht]
\centering
\IfFileExists{figures/chronosrl/robot_summary_final.pdf}{\includegraphics[width=\linewidth]{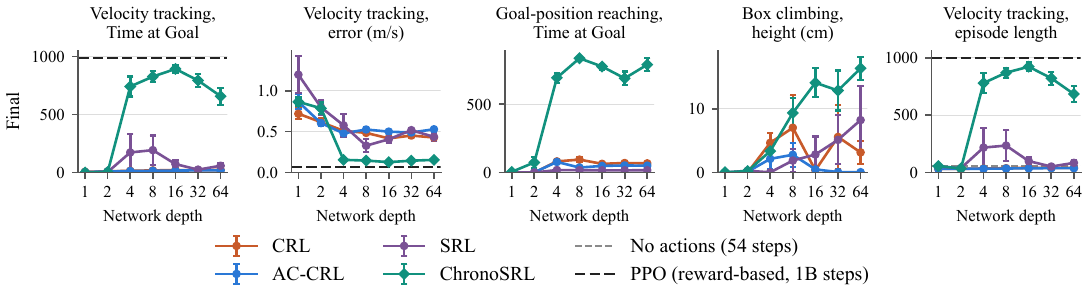}}{\fbox{Final robot values pending}}
\caption{At the end of training and from four layers on, \ours{} stays at velocity and position goals longest and tracks the forward command most accurately, while the contrastive baselines end their velocity-tracking episodes even before the dashed line, at which a robot that sends no actions falls.
Points show the mean of the last five evaluations over four seeds with one standard error.}
\label{fig:robots_final}
\end{figure}

The learning curves in Figure~\ref{fig:robot_curves} show that the box curricula of \gls{crl} and \gls{accrl} rise fastest, with single runs reaching boxes of up to $25$ centimeters early in training.
From four layers on, however, these curricula later fall back at every depth.
The curriculum of \ours{} rises more slowly, but in $16$ of its $20$ runs with four or more layers it is still rising at the end of training.
\begin{figure}[ht]
\centering
\IfFileExists{figures/chronosrl/robot_curves.pdf}{\includegraphics[width=\linewidth]{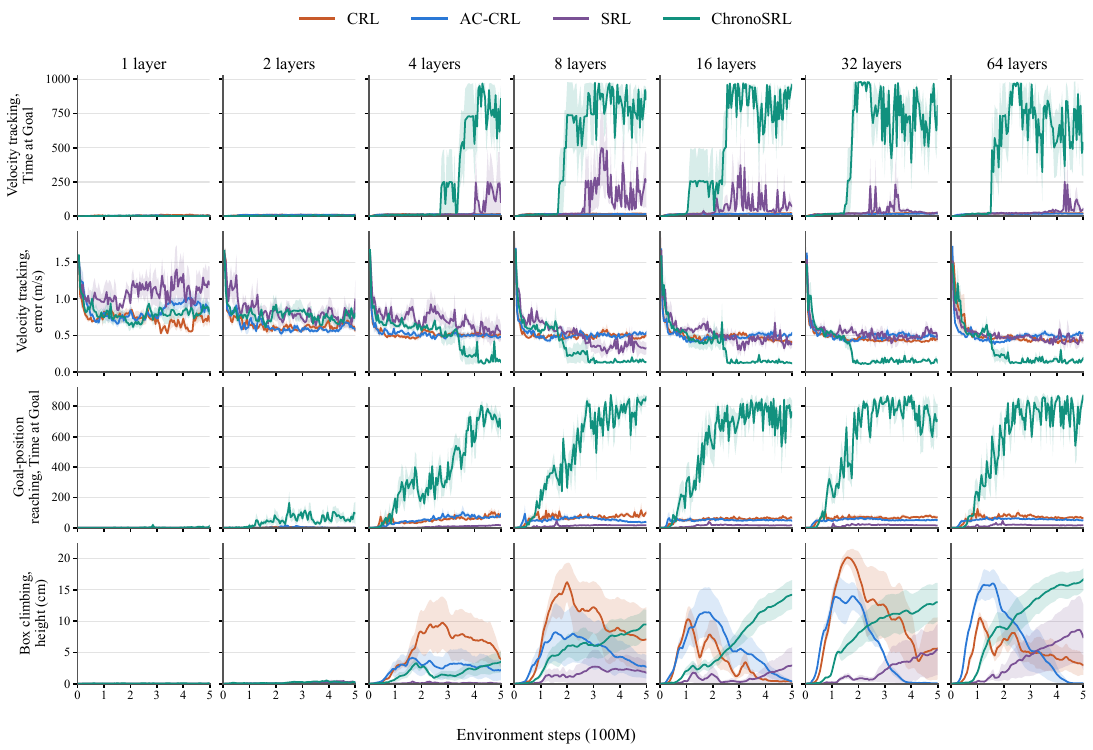}}{\fbox{Robot learning curves pending}}
\caption{From four layers on, \ours{} learns to stay at velocity and position goals, and its box curriculum keeps rising until the end of training, while those of the contrastive baselines peak early and fall back.
Rows show the four metrics of Figure~\ref{fig:robotics} and columns the network depth, and each curve shows the mean over four seeds with one standard error over $500$ million steps.}
\label{fig:robot_curves}
\end{figure}

\section{When hindsight goals are already satisfied}
\label{app:boundary}
\begin{figure}[ht]
\centering
\includegraphics[width=\linewidth]{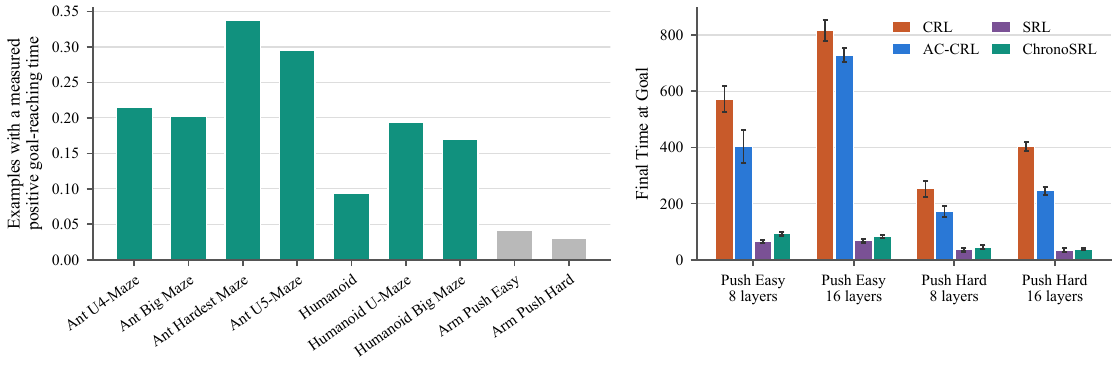}
\caption{Temporal supervision needs trajectories that make progress toward their relabeled goals.
The left panel shows the fraction of relabeled examples of \ours{} with a positive goal-reaching time in the final training batches, averaged over all runs of a task, and the right panel shows the final \ToG{} on the two arm tasks.
\gls{srl} uses its published arm settings, namely a dwell of $k=20$ on Arm Push Easy and a discount of $0.99$ on Arm Push Hard.
Bars show the mean over four seeds with one standard error.}
\label{fig:arms}
\end{figure}
In the arm manipulation tasks of JaxGCRL, the goal is the position of an object, and this object rarely moves before the arm has learned to push it.
As a result, most relabeled goals are already reached at the sampled state.
Indeed, the left panel of Figure~\ref{fig:arms} shows that only $3$ to $4\%$ of the relabeled examples have a positive goal-reaching time, compared to $9$ to $34\%$ in the locomotion and navigation tasks.
With so few positive goal-reaching times, the time loss and the hazard likelihood receive few informative targets.
Consistent with this, the right panel of Figure~\ref{fig:arms} shows that the survival critics of \ours{} and \gls{srl} stay below a \ToG{} of $100$, whereas the contrastive critics of \gls{crl} and \gls{accrl} reach several hundred.
Goals that also contain the position of the gripper, rather than that of the object alone, could make progress visible long before the object moves.

\end{document}